%% file: main.tex
\documentclass{article}
\usepackage{fullpage}
\title{A Block Coordinate Ascent Algorithm \\
for Mean-Variance Optimization}

\author{
  Bo Liu\thanks{Equal contribution. Corresponding to: \texttt{boliu@auburn.edu}. \newline \newline
32nd Conference on Neural Information Processing Systems (NIPS 2018), Montreal, Canada. Copyright 2018 by the authors.} \\
  Auburn University\\
  \texttt{boliu@auburn.edu} \\
  \and
  Tengyang Xie\footnotemark[1]\\
  UMass Amherst\\
  \texttt{txie@cs.umass.edu} \\
  \and
  Yangyang Xu\\
  Rensselaer Polytechnic Institute\\
  \texttt{xuy21@rpi.edu} \\
  \and
  Mohammad Ghavamzadeh\\
  Facebook AI Research\\
  \texttt{mgh@fb.com} \\
  \and
  Yinlam Chow\\
  Google DeepMind\\
  \texttt{yinlamchow@google.com} \\
  \and
  Daoming Lyu\\
  Auburn University\\
  \texttt{daoming.lyu@auburn.edu} \\
  \and
  Daesub Yoon\\
  ETRI\\
  \texttt{eyetracker@etri.re.kr} \\}

\date{}

\usepackage{algorithm}
\usepackage{algorithmic}

\usepackage{epsf}

\usepackage{amssymb}
\usepackage{wrapfig}
\usepackage{mathtools}

\newcounter{RomanNumber}
\newcommand{\RomanNum}[1]{\setcounter{RomanNumber}{#1}\Roman{RomanNumber}}

\usepackage{graphicx}
\usepackage{geometry}

\usepackage{amsmath}
\usepackage{amsfonts}
\usepackage{xcolor}
\definecolor{darkblue}{HTML}{000080}

\usepackage{multirow}

\usepackage{url}
\usepackage[colorlinks=true, linkcolor=black, citecolor=darkblue]{hyperref}

\usepackage{mathtools}
\usepackage{mathrsfs}
\mathtoolsset{showonlyrefs}
\usepackage{natbib}
\setcitestyle{authoryear} % authoryear, numbers, super
\setcitestyle{square}
\setcitestyle{semicolon} % semicolon, comma

\usepackage{amsthm}

\newcounter{thm_counter}
\newcounter{lem_counter}

\newcounter{cor_counter}
\newcounter{ass_counter}

\newcounter{rmk_counter}

\newtheorem{theorem}[thm_counter]{Theorem}%[section]
\newtheorem{lemma}[lem_counter]{Lemma}%[Lemma]
\newtheorem{corollary}[cor_counter]{Corollary}
\newtheorem{assumption}[ass_counter]{Assumption}

\newtheorem{remark}[rmk_counter]{Remark}

\usepackage[finalold]{trackchanges}

\addeditor{Place_Holder} % blue note is easy to be confused by citations
\addeditor{mgh}

\addeditor{lb}

\addeditor{ty}

\addeditor{YY}% Dr. Yangyang Xu

\usepackage{microtype}
\usepackage{graphicx}
\usepackage[labelformat=simple]{subcaption}

\usepackage{booktabs} % for professional tables

\usepackage{hyperref}

\begin{document}
  
\maketitle
% \linenumbers

\setcounter{footnote}{0}




\begin{abstract}
Risk management in dynamic decision problems is a primary concern in many fields, including financial investment, autonomous driving, and healthcare. The mean-variance function is one of the most widely used objective functions in risk management due to its simplicity and interpretability. Existing algorithms for mean-variance optimization are based on multi-time-scale stochastic approximation, whose learning rate schedules are often hard to tune, and have only asymptotic convergence proof. In this paper, we develop a model-free policy search framework for mean-variance optimization with finite-sample error bound analysis (to local optima). Our starting point is a reformulation of the original mean-variance function with its Legendre-Fenchel dual, from which we propose a stochastic block coordinate ascent policy search algorithm. Both the asymptotic convergence guarantee of the last iteration's solution and the convergence rate of the randomly picked solution are provided, and their applicability is demonstrated on several benchmark domains.
\end{abstract}

%%%%%%%%%%%%%%%%%%%%%%%%%%%%%
%%%%%%%%%%%%%%%%%%%%%%%%%%%%%
%%%%%%%%%%%%%%%%%%%%%%%%%%%%%
\section{Introduction}
\label{sec:intro}

Risk management plays a central role in sequential decision-making problems, common in fields such as portfolio management~\citep{lai2011mean}, autonomous driving~\citep{risk:autonomous:maurer2016}, and healthcare~\citep{risk:american2011}. A common risk-measure is the variance of the expected sum of rewards/costs and the mean-variance trade-off function~\citep{sobel1982variance,saferl:mannor2011} is one of the most widely used objective functions in risk-sensitive decision-making. Other risk-sensitive objectives have also been studied, for example,~\citet{saferl:borkar2002q} studied exponential utility functions,~\citet{saferl:castro2012} experimented with the Sharpe Ratio measurement,~\citet{Chow18RC} studied value at risk (VaR) and mean-VaR optimization,~\citet{saferl:cvar:chow2014},~\citet{Tamar15OC}, and~\citet{Chow18RC} investigated conditional value at risk (CVaR) and mean-CVaR optimization in a static setting, and~\citet{saferl:tamar2015coherent} investigated coherent risk for both linear and nonlinear system dynamics. 
%Chow~\citep{saferl:cvar:chow2015} showed the equivalence between a coherent risk-sensitive objective and a robust MDP. 
Compared with other widely used performance measurements, such as the Sharpe Ratio and CVaR, the mean-variance measurement has explicit interpretability and computational advantages~\citep{meanvar:markowitz2000mean,meanvar:li2000optimal}. For example, the Sharpe Ratio tends to lead to solutions with less mean return~\citep{saferl:castro2012}. 
Existing mean-variance reinforcement learning (RL) algorithms~\citep{saferl:castro2012,saferl:prashanth2013,Prashanth16VC} often suffer from heavy computational cost, slow convergence, and difficulties in tuning their learning rate schedules. Moreover, all their analyses are asymptotic and no rigorous finite-sample complexity analysis has been reported. 
{Recently, \citet{pmlr-v75-dalal18a} provided a general approach to compute finite sample analysis in the case of linear multiple time scales stochastic approximation problems. However, existing multiple time scales algorithms like \citep{saferl:castro2012} consist of nonlinear term in its update, and cannot be analyzed via the method in \citet{pmlr-v75-dalal18a}.}
All these make it difficult to use them in real-world problems. The goal of this paper is to propose a mean-variance optimization algorithm that is both computationally efficient and has finite-sample analysis guarantees. This paper makes the following contributions: {\bf 1)} We develop a computationally efficient RL algorithm for mean-variance optimization.  By reformulating the mean-variance function with its Legendre-Fenchel dual~\citep{boyd}, we propose a new formulation for mean-variance optimization and use it to derive a computationally efficient algorithm that is based on stochastic cyclic block coordinate descent. 
%This framework also enables using any off-the-shelf stochastic nonconvex block coordinate descent solver~\citep{cd:bsg:xu2015, cd:sbmd:dang2015}, which has the potential for acceleration and variance reduction. 
{\bf 2)} We provide the sample complexity analysis of our proposed algorithm. This result is novel because although cyclic block coordinate descent algorithms usually have empirically better performance than randomized block coordinate descent algorithms, yet almost all the reported analysis of these algorithms are asymptotic~\citep{cd:bsg:xu2015}. 

Here is a roadmap for the rest of the paper. Section~\ref{sec:backgrounds} offers a brief background on risk-sensitive RL and stochastic variance reduction. In Section~\ref{sec:alg}, the problem is reformulated using the Legendre-Fenchel duality and a novel algorithm is proposed based on stochastic block coordinate descent. Section~\ref{sec:theory} contains the theoretical analysis of the paper that includes both asymptotic convergence and finite-sample error bound. The experimental results of Section~\ref{sec:experimental} validate the effectiveness of the proposed algorithms.

%%%%%%%%%%%%%%%%%%%%%%%%%%%%%
%%%%%%%%%%%%%%%%%%%%%%%%%%%%%
%%%%%%%%%%%%%%%%%%%%%%%%%%%%%

\section{Backgrounds}
\label{sec:backgrounds}

This section offers a brief overview of risk-sensitive RL, including the objective functions and algorithms. We then introduce block coordinate descent methods. Finally, we introduce the Legendre-Fenchel duality, the key ingredient in formulating our new algorithms.

%%%%%%%%%%%%%%%%%%%%%%%%%%%%%

\subsection{Risk-Sensitive Reinforcement Learning}
\label{sec:riskrl}

Reinforcement Learning (RL)~\citep{sutton-barto:book} is a class of learning problems in which an agent interacts with an unfamiliar, dynamic, and stochastic environment, where the agent's goal is to optimize some measures of its long-term performance. This interaction is conventionally modeled as a Markov decision process (MDP), defined as the tuple $({\mathcal{S},\mathcal{A},P_0,P_{ss'}^{a},r,\gamma})$, where $\mathcal{S}$ and $\mathcal{A}$ are the sets of states and actions, $P_0$ is the initial state distribution, $P_{ss'}^{a}$ is the transition kernel that specifies the probability of transition from state $s\in\mathcal{S}$ to state $s'\in\mathcal{S}$ by taking action $a\in\mathcal{A}$, $r(s,a):\mathcal{S}\times\mathcal{A}\to\mathbb{R}$ is the reward function bounded by $R_{\max}$, and $0\leq\gamma<1$ is a discount factor. A \textit{parameterized stochastic policy} $\pi_\theta(a|s):\mathcal{S}\times\mathcal{A}\to\left[{0,1}\right]$ is a probabilistic mapping from states to actions, where $\theta$ is the tunable parameter and $\pi_\theta(a|s)$ is a differentiable function w.r.t.~$\theta$. %Different probability distributions over $\mathcal{A}$ are associated with each $s\in \mathcal{S}$ for different values of $\theta$.

One commonly used performance measure for policies in {\em episodic} MDPs is the {\em return} or cumulative sum of rewards from the starting state, i.e.,~$R = \sum\nolimits_{k = 1}^{{\tau}} {r({s_k},{a_k})}$, 
where $s_1\sim P_0$ and $\tau$ is the first passage time to the recurrent state $s^*$~\citep{puterman,saferl:castro2012}, and thus, $\tau \coloneqq \min \{ k > 0\;|\;{s_k} = {s^*}\}$.  
In risk-neutral MDPs, the algorithms aim at finding a near-optimal policy that maximizes the expected sum of rewards $J(\theta) \coloneqq \mathbb{E}_{\pi_\theta}[R] = \mathbb{E}_{\pi_\theta}\big[\sum_{k = 1}^{{\tau}}{r({s_k},{a_k})}\big]$. We also define the square-return $M(\theta ) \coloneqq {\mathbb{E}_{{\pi _\theta }}}[{R^2}] = {\mathbb{E}_{{\pi _\theta }}}\Big[{\big(\sum_{k = 1}^{\tau}  {r({s_k},{a_k})}\big)^2}\Big]$. In the following, we sometimes drop the subscript $\pi_\theta$ to simplify the notation.
%
%
% In risk-sensitive reinforcement learning and control, the optimization criterion is transformed into the \textit{mean-variance} trade-off ${J_\lambda(\theta) }$, defined as 
% $
% {J_\lambda }(\theta ): = J(\theta ) - \lambda {\rm{Var}}(R),
% $
% where $\lambda >0$ is the risk-control parameter, and ${\rm {Var}}(R) = M(\theta) - J^2(\theta)$ measures the variance of $R$. 
% %
% % So ${J_\lambda }(\theta )$ can be formulated as
% % \begin{equation}
% % {J_\lambda }(\theta ): = J(\theta ) - \lambda (M(\theta ) - {J^2}(\theta )),\quad \lambda  > 0
% % \label{eq:regobj}
% % \end{equation}
% %
% So the optimization problem is maximizing $J(\theta)$ with variance constraint 
% \begin{align}
% \max_{\theta} J(\theta )\quad
%  \text{s.t.} ~ \text{Var}(R) \leq \eta,
% \end{align}
% where $\eta > 0$ is a pre-defined constant to control the magnitude of the variance~\citep{saferl:castro2012}.
% Using Lagrange multiplier method, we can transfer the optimization problem above to maximizing a unconstrained objective function ${J_\lambda }(\theta )$, which is $J(\theta)$  with variance penalty, i.e.,
% \begin{align}
% {J_\lambda }(\theta ): = J(\theta ) - \lambda (M(\theta ) - {J^2}(\theta )),\quad \lambda  > 0.
% \label{eq:regobj}
% \end{align}
%

In risk-sensitive mean-variance optimization MDPs, the objective is often to maximize $J(\theta)$ with a variance constraint, i.e.,
\begin{equation}
\begin{aligned}
\max_{\theta} \quad & J(\theta) = \mathbb{E}_{\pi_\theta}[R] \\
\text{s.t.} \quad & \text{Var}_{\pi_\theta}(R) \leq \zeta,
\label{eq:opt-prob1}
\end{aligned}
\end{equation}
where $\text{Var}_{\pi_\theta}(R) = M(\theta) - J^2(\theta)$ measures the variance of the return random variable $R$, and $\zeta > 0$ is a given risk parameter~\citep{saferl:castro2012,saferl:prashanth2013}. Using the Lagrangian relaxation procedure~\citep{bertsekas:npbook}, we can transform the optimization problem~\eqref{eq:opt-prob1} to maximizing the following unconstrained objective function:% ${J_\lambda }(\theta )$, which is the \textit{mean-variance} trade-off, i.e.,
\begin{align}
{J_\lambda }(\theta ) \coloneqq & \mathbb{E}_{\pi_\theta}[R] - \lambda\big(\text{Var}_{\pi_\theta}(R) - \zeta) \\
= & J(\theta ) - \lambda\big(M(\theta ) - J(\theta)^2 - \zeta\big).
\label{eq:regobj}
\end{align}
It is important to note that the mean-variance objective function is \textit{NP-hard} in general~\citep{saferl:mannor2011}. The main reason for the hardness of this optimization problem is that although the variance satisfies a Bellman equation~\citep{sobel1982variance}, unfortunately, it lacks the monotonicity property of dynamic programming (DP), and thus, it is not clear how the related risk measures can be optimized by standard DP algorithms~\citep{sobel1982variance}. 

The existing methods to maximize the objective function~\eqref{eq:regobj} are mostly based on stochastic approximation that often converge to an equilibrium point of an ordinary differential equation (ODE)~\citep{borkar:book}. For example,~\citet{saferl:castro2012} proposed a policy gradient algorithm, a two-time-scale stochastic approximation, to maximize~\eqref{eq:regobj} for a fixed value of $\lambda$ (they optimize over $\lambda$ by selecting its best value in a finite set), while the algorithm in~\citet{saferl:prashanth2013} to maximize~\eqref{eq:regobj} is actor-critic and is a three-time-scale stochastic approximation algorithm (the third time-scale optimizes over $\lambda$). These approaches suffer from certain drawbacks: {\bf 1)} Most of the analyses of ODE-based methods are asymptotic, with no sample complexity analysis. {\bf 2)} It is well-known that multi-time-scale approaches are sensitive to the choice of the stepsize schedules, which is a non-trivial burden in real-world problems. {\bf 3)} The ODE approach does not allow extra penalty functions. Adding penalty functions can often strengthen the robustness of the algorithm, encourages sparsity and incorporates prior knowledge into the problem~\citep{esl:ElementsStatisticalLearning}. 

%%%%%%%%%%%%%%%%%%%%%%%%%%%%%
\subsection{Coordinate Descent Optimization}
\label{sec:bcd}

Coordinate descent (CD)\footnote{Note that since our problem is maximization, our proposed algorithms are block coordinate \emph{ascent}.} and the more general block coordinate descent (BCD) algorithms solve a minimization problem by iteratively updating variables along coordinate directions or coordinate hyperplanes~\citep{cd:overview:wright2015}. At each iteration of BCD, the objective function is (approximately) minimized w.r.t.~a coordinate or a block of coordinates by fixing the remaining ones, and thus, an easier lower-dimensional subproblem needs to be solved. A number of comprehensive studies on BCD have already been carried out, such as~\citet{cd:bcd:luo1992} and~\citet{cd:sbcd:nesterov2012} for convex problems, and~\citet{cd:bcd:tseng2001},~\citet{xu2013block}, and~\citet{razaviyayn2013unified} for nonconvex cases (also see~\citealt{cd:overview:wright2015} for a review paper). For stochastic problems with a block structure,~\citet{cd:sbmd:dang2015} proposed stochastic block mirror descent (SBMD) by combining BCD with stochastic mirror descent~\citep{MID:2003,nemirovski2009robust}. Another line of research on this topic is block stochastic gradient coordinate descent (BSG)~\citep{cd:bsg:xu2015}. The key difference between SBMD and BSG is that at each iteration, SBMD randomly picks one block of variables to update, while BSG cyclically updates all block variables. 
% \noteb{Mohammad, I commented out the following part and reduce its size, since our focus is not comparing BSG and SBMD. what do u think? }
%% comment
%Theoretically, SBMD generates an unbiased estimation of the stochastic approximation of a partial gradient, and thus, has asymptotic convergence as well as finite-sample analysis results for both convex and nonconvex problems. Conversely, BSG produces a biased stochastic gradient approximation. To be more accurate, in BSG, the stochastic gradient approximation of the first block of variables is unbiased and all the rest are biased, which creates problems for its finite-sample analysis. 
%However, BSG has been shown that numerically outperforms SBDM. This is mainly due to the fact that BSG has lower per-epoch complexity and its cyclic-style update leads to a Gauss-Seidel type acceleration.
%% end of comment

In this paper, we develop mean-variance optimization algorithms based on both nonconvex stochastic BSG and SBMD. Since it has been shown that the BSG-based methods usually have better empirical performance than their SBMD counterparts, the main algorithm we report, analyze, and evaluate in the paper is BSG-based. We report our SBMD-based algorithm in Appendix~\ref{sec:rcpg-sga} and use it as a baseline in the experiments of Section~\ref{sec:experimental}. The finite-sample analysis of our BSG-based algorithm reported in Section~\ref{sec:theory} is novel because although there exists such analysis for convex stochastic BSG methods~\citep{cd:bsg:xu2015}, we are not aware of similar results for their nonconvex version to the best our knowledge.

%\subsection{Fenchel Duality}
%
%Next, we introduce the Legendre-Fenchel duality. A closed and convex function $f(z)$ can be represented by its Legendre-Fenchel dual formulation~\citep{boyd}:
%$
%f(z) = \mathop {\max }_y ({z^\top}y - {f^*}(y)),
%%\label{eq:fenchel}
%$
%where $f^*$ is the convex conjugate function~\citep{boyd}. Especially, for $z \in \mathbb{R}$, the convex conjugate function of a power function $f(z)={\frac {1}{p}}|z|^{p},\,1<p<\infty$ is $f^{* }\left(z\right)={\frac {1}{q}}|x|^{q},\,1<q<\infty$, where ${\frac {1}{p}}+{\frac {1}{q}}=1$. It is easy to obtain when $p=2$ that 
%$
%f(z) = {f^{*}}(z) = \frac{1}{2}{z^2}
%$.  
%So for the function $z^2, z \in \mathbb{R}$, we have the Legendre-Fenchel dual formulation as follows
%$
%{z^2} = \mathop {\max }_{y\in \mathbb{R}} (2zy - {y^2}).
%$
%%\label{eq:fenchel-square}

%%%%%%%%%%%%%%%%%%%%%%%%%%%%%
%%%%%%%%%%%%%%%%%%%%%%%%%%%%%
%%%%%%%%%%%%%%%%%%%%%%%%%%%%%
\section{Algorithm Design}
\label{sec:alg}
In this section, we first discuss the difficulties of using the regular stochastic gradient ascent to maximize the mean-variance objective function. We then propose a new formulation of the mean-variance objective function that is based on its Legendre-Fenchel dual and derive novel algorithms that are based on the recent results in stochastic nonconvex block coordinate descent. We conclude this section with an asymptotic analysis of a version of our proposed algorithm. %introduce the mean-variance function formulation and explain why it is difficult to use the regular stochastic gradient for maximization. Then we offer a reformulation of the objective function, and a novel algorithm based on recent stochastic nonconvex block coordinate descent %achievements with in-depth theoretical analysis.

%%%%%%%%%%%%%%%%%%%%%%%%%%%%%

\subsection{Problem Formulation}
\label{subsec:formulation}

In this section, we describe why the vanilla stochastic gradient cannot be used to maximize $J_\lambda(\theta)$ defined in Eq.~\eqref{eq:regobj}. Taking the gradient of $J_\lambda(\theta)$ w.r.t.~$\theta$, we have
\begin{align}
\nabla_\theta {J_\lambda }({\theta _t}) = & {\nabla _\theta }J({\theta _t}) - \lambda {\nabla _\theta }{\rm{Var}}(R) \\
= & {\nabla _\theta }J({\theta _t}) - \lambda\big(\nabla_{\theta} M(\theta) - 2J({\theta})\nabla_{\theta} J({\theta})\big).
\label{objgradient}
\end{align}
Computing $\nabla_\theta {J_\lambda }({\theta _t})$ in~\eqref{objgradient} involves computing three quantities: $\nabla_{\theta} J({\theta}), \nabla_{\theta} M({\theta})$, and $J({\theta})\nabla_{\theta} J({\theta})$. We can obtain unbiased estimates of $\nabla_{\theta}J(\theta)$ and $\nabla_{\theta}M(\theta)$ from a single trajectory generated by the policy $\pi_\theta$ %with $\mathbb{E}_t[\tau_t]=\tau$.} 
using the likelihood ratio method~\citep{reinforce:williams1992}, as $\nabla_{\theta}J({\theta})=\mathbb{E}[{R_t}\omega_{t}(\theta)]$ and $\nabla_{\theta}M({\theta})=\mathbb{E}[R_t^2\omega_{t}(\theta)]$. Note that $R_t$ is the cumulative reward of the $t$-th episode, i.e.,~$R_t = \sum\nolimits_{k=1}^{{\tau_t}} r_k$, which is possibly a \textit{nonconvex} function, and $\omega _t(\theta ) = \sum _{k=1}^{{\tau_t}}{\nabla_{\theta}\ln{\pi_{\theta}}({a_{k}}|{s_{k}})}$ is the likelihood ratio derivative. In the setting considered in the paper, an episode is the trajectory between two visits to the recurrent state $s^*$. For example, the $t$-th episode refers to the trajectory between the ($t$-1)-th and the $t$-th visits to $s^*$. We denote by $\tau_t$ the length of this episode.

However, it is not possible to compute an unbiased estimate of $J(\theta)\nabla_{\theta} J(\theta)$ without having access to a generative model of the environment that allows us to sample at least two next states $s'$ for each state-action pair $(s, a)$. As also noted by~\citet{saferl:castro2012} and~\citet{saferl:prashanth2013}, computing an unbiased estimate of $J(\theta)\nabla_{\theta} J(\theta)$ requires double sampling (sampling from two different trajectories), and thus,  cannot be done using a single trajectory. To circumvent the double-sampling problem, these papers proposed multi-time-scale stochastic approximation algorithms, the former a policy gradient algorithm and the latter an actor-critic algorithm that uses simultaneous perturbation methods~\citep{Bhatnagar13SR}. However, as discussed in Section~\ref{sec:riskrl}, multi-time-scale stochastic approximation approach suffers from several weaknesses such as no available finite-sample analysis and difficult-to-tune stepsize schedules. To overcome these weaknesses, we reformulate the mean-variance objective function and use it to present novel algorithms with in-depth analysis in the rest of the paper.

%%%%%%%%%%%%%%%%%%%%%%%%%%%%%
\subsection{Block Coordinate Reformulation}

%Fenchel duality is a powerful tool for avoiding double sampling.
%Recent research~\citep{liu2015uai,vr:saddle:pe:du2017} has formulated a class of reinforcement learning problems as convex-concave saddle-point problems by introducing the Legendre-Fenchel dual to avoid the double sampling problem in off-policy policy evaluation.
In this section, we present a new formulation for $J_\lambda(\theta)$ that is later used to derive our algorithms and do not suffer from the double-sampling problem in estimating $J(\theta)\nabla_{\theta} J(\theta)$. We begin with the following lemma.
\begin{lemma}
\label{lem:fenchel}
For the quadratic function $f(z)=z^2,\;z\in\mathbb{R}$, we define its Legendre-Fenchel dual as $f(z)=z^2=\mathop{\max}_{y\in \mathbb{R}} (2zy - {y^2})$.
\end{lemma}
This is a special case of the Lengendre-Fenchel duality~\citep{boyd} that has been used in several recent RL papers (e.g.,~\citealt{liu2015uai,vr:saddle:pe:du2017,liubo:jair:2018}). Let ${F_\lambda }(\theta ) \coloneqq {\left({J(\theta ) + \frac{1}{{2\lambda }}} \right)^2} - M(\theta )$, which follows ${F_\lambda }(\theta )=\frac{{{J_\lambda }(\theta)}}{\lambda } + \frac{1}{{4{\lambda ^2}}} - \zeta$.  Since $\lambda>0$ is a constant, maximizing $J_\lambda(\theta)$ is equivalent to maximizing ${F_\lambda }(\theta )$. 
% \note[YY]{Should this be $F_\lambda(\theta)$?-FIXED-BO}  
Using Lemma~\ref{lem:fenchel} with $z = J(\theta ) + \frac{1}{2\lambda}$, we may reformulate ${F_\lambda }(\theta )$ as
\begin{align}
\label{eq:000}
{F_\lambda }(\theta) = \mathop {\max }_y \Big( {2y\big(J(\theta) + \frac{1}{{2\lambda }}\big) - {y^2}} \Big) - M(\theta).
\end{align}
%
%where $y$ is a scalar variable. 
% The general form of the Legendre-Fenchel dual can be found in~\citep{boyd}, and is omitted here due to space constraints. 
Using~\eqref{eq:000}, the maximization problem $\mathop {\max }_{\theta} {F_\lambda }(\theta )$ is equivalent to
\begin{equation}
\begin{aligned}
\mathop {\max}_{\theta ,y} \qquad & {\hat f_\lambda }(\theta ,y), \\
\text{where} \qquad & {{{\hat f}_\lambda }(\theta ,y) \coloneqq 2y\big(J(\theta ) + \frac{1}{{2\lambda }}\big) - {y^2} - M(\theta )}.
\label{eq:max-alt}
\end{aligned}
\end{equation}
Our optimization problem is now formulated as the standard nonconvex coordinate ascent problem~\eqref{eq:max-alt}. We use three stochastic solvers to solve~\eqref{eq:max-alt}: SBMD method~\citep{cd:sbmd:dang2015}, BSG method~\citep{cd:bsg:xu2015}, and the vanilla stochastic gradient ascent (SGA) method~\citep{nemirovski2009robust}. We report our BSG-based algorithm in Section~\ref{subsec:MVPG} and leave the details of the SBMD and SGA based algorithms to Appendix~\ref{sec:rcpg-sga}. In the following sections, we denote by $\beta_t^\theta$ and $\beta_t^y$ the stepsizes of $\theta$ and $y$, respectively, and by the subscripts $t$ and $k$ the episode and time-step numbers. 

%%%%%%%%%%%%%%%%%%%%%%%%%%%%%
\subsection{Mean-Variance Policy Gradient}
\label{subsec:MVPG}

We now present our main algorithm that is based on a block coordinate update to maximize~\eqref{eq:max-alt}. Let ${g}_t^\theta$ and ${g}_t^y$ be block gradients and $\tilde{g}_t^\theta$ and $\tilde{g}_t^y$ be their sample-based estimations defined as 
\begin{align}
\label{def:tildeg_y}
g^y_t = \mathbb{E}[\tilde{g}^y_t] = 2 J(\theta_t) + \frac{1}{\lambda } - 2y_t \quad , \quad & \tilde{g}^y_t = 2 R_t + \frac{1}{\lambda } - 2y_t,\\
\label{def:tildeg_theta}
g^\theta_t = \mathbb{E}[\tilde g^\theta_t] = 2y_{t + 1}\nabla_{\theta}{J(\theta_t)} - \nabla_{\theta} M(\theta_t) \quad , \quad & \tilde{g}^\theta_t = \left( {2y_{t + 1}{R_t} - {{({R_t})}^2}} \right){\omega _t}({\theta _t}).
\end{align}
The block coordinate updates are
\begin{align}
\label{eq:updatelaw}
{y_{t+1}} = & {y_{t}}+{\beta^y_{t}}{\tilde{g}_t^y}, \\
{\theta_{t+1}} = & {\theta_{t}}+{\beta^\theta_{t}}{\tilde{g}_t^\theta}.
\end{align}
To obtain unbiased estimates of $g^y_t$ and $g^\theta_t$, we shall update $y$ (to obtain $y_{t+1}$) prior to computing $g^\theta_t$ at each iteration. Now it is ready to introduce the Mean-Variance Policy Gradient (\textbf{MVP}) Algorithm~\ref{alg:mvp}.
\begin{algorithm}[htb]
\caption{Mean-Variance Policy Gradient (\textbf{MVP})}
\label{alg:mvp}
\centering
\begin{algorithmic}[1]
\STATE {\bfseries Input:} Stepsizes $\{\beta_t^\theta\}$ and $\{\beta_t^y\}$, and number of iterations $N$ \\ 
$\quad$\textbf{Option I:} $\{\beta_t^\theta\}$ and $\{\beta_t^y\}$ satisfy the Robbins-Monro condition \\
$\quad$\textbf{Option II:} $\beta_t^\theta$ and $\beta_t^y$ are set to be constants
% such that $\sum_t^\infty \beta_t = \infty$ and $\sum_t^\infty \beta_t^2 < \infty$. 
\FOR {episode $t=1,\dotsc,N$}
\STATE Generate the initial state $s_1\sim P_0$ 
\WHILE{$s_k \neq s^*$}
\STATE Take the action $a_k \sim \pi_{\theta_t}(a|s_k)$ and observe the reward $r_k$ and next state $s_{k+1}$
\ENDWHILE
% \FOR {time step $k=1,\dotsc,{\tau_t}$} 
% \STATE Compute $a_k \sim \pi_{\theta_t}(a|s_k)$, observe $r_k, s_{k+1}$.
% \ENDFOR
\STATE Update the parameters
\begin{align}
% {R_t} &= \sum_{k = 1}^{\tau_t}  {{r_k}} \qquad\qquad\qquad\qquad\qquad\quad {\omega _t}(\theta_t)= \sum _{k=1}^{{\tau_t}}{\nabla_\theta \ln{\pi_{\theta_t}}({a_{k}}|{s_{k}})} \\
% {y_{t+1}}
% &= {y_t} + {\beta _t^y}\Big(2 R_t  + \frac{1}{\lambda } - 2y_t\Big) \qquad\quad {\theta_{t+1}}= {\theta _t} + {\beta _t^\theta}\Big( {2y_{t+1}{R_t } - {{({R_t })}^2}} \Big){\omega _t}({\theta _t})
{R_t} = & \sum_{k = 1}^{\tau_t}  {{r_k}} \\
{\omega _t}(\theta_t) = & \sum _{k=1}^{{\tau_t}}{\nabla_\theta \ln{\pi_{\theta_t}}({a_{k}}|{s_{k}})} \\
{y_{t+1}} = & {y_t} + {\beta _t^y}\left(2 R_t  + \frac{1}{\lambda } - 2y_t\right)\\
{\theta_{t+1}} = & {\theta _t} + {\beta _t^\theta}\left( {2y_{t+1}{R_t } - {{({R_t })}^2}} \right){\omega _t}({\theta _t})
% {y_{t+1}}
% &= {y_t} + {\beta _t}(2\hat J(\theta) + \frac{1}{\lambda } - 2y_t)
\end{align}
\ENDFOR
%\STATE {\textbf{Output}} $\bar{x}_N = x_N = [\theta_N,y_N]^\top$.
\STATE {\bfseries Output} $\bar{x}_N$: \\
$\quad$\textbf{Option I:} Set  $\bar{x}_N = x_N = [\theta_N,y_N]^\top$ \\
$\quad$\textbf{Option II:} Set $\bar{x}_N = x_z= [\theta_z,y_z]^\top$, where $z$ is uniformly drawn from $\{ 1,2, \ldots ,N\}$
\end{algorithmic}
\end{algorithm}
Before presenting our theoretical analysis, we first introduce the assumptions needed for these results. 
% \notet{We may also need the assumptions bounded gradient and Lipschitz smoothness here}
% \noteb{only bounded gradient is necessary, and can be combined together with Ass1-FIXED-BO}.
%%%%%%%%%%%%%%
\begin{assumption}[\textbf{Bounded Gradient and Variance}]
\label{asm3}
There exist constants $G$ and $\sigma$ such that
\begin{align}
\|\nabla_{y}\hat f_{\lambda}(x)\|_2 \leq G, ~ & \|\nabla_{\theta} \hat f_{\lambda} (x)\|_2 \leq G, \\
\mathbb{E}[\|\Delta^y_t\|_2^2] \leq {\sigma}^2, ~ & \mathbb{E}[\|\Delta^\theta_t\|_2^2] \leq {\sigma}^2,
% &\|\mathbb{E}[\Delta^\theta_t |\{ {R_i},{\omega _i}\} _{i = 1}^{t - 1}]\|_2 \leq A\cdot \beta_t^{\max}.
\end{align}
for any $t$ and $x$, where $\|\cdot\|_2$ denotes the Euclidean norm, ${\Delta}^y_t \coloneqq \tilde{g}^y_t - {g}^y_t$ and ${\Delta}^\theta_t \coloneqq \Tilde{g}^\theta_t - {g}^\theta_t$.
\label{asm:var}
\end{assumption}
Assumption~\ref{asm:var} is standard in nonconvex coordinate descent algorithms~\citep{cd:bsg:xu2015,cd:sbmd:dang2015}. We also need the following assumption that is standard in the policy gradient literature.
\begin{assumption}[\textbf{Ergodicity}]
The Markov chains induced by all the policies generated by the algorithm are ergodic, i.e.,~irreducible, aperiodic, and recurrent.
\label{asm:erg}
\end{assumption}
In practice, we can choose either Option I with the result of the final iteration as output or Option II with the result of a randomly selected iteration as output. In what follows in this section, we report an asymptotic convergence analysis of MVP with Option I, and in Section~\ref{sec:theory}, we derive a finite-sample analysis of MVP with Option II. 
\begin{theorem}[\textbf{Asymptotic Convergence}]
\label{thm:asyn}
% (\textbf{Asymptotic Convergence})
Let $\big\{x_t=(\theta_t, y_t)\big\}$ be the sequence of the outputs generated by Algorithm~\ref{alg:mvp} with Option I. If $\{\beta_t^\theta\}$ and $\{\beta_t^y\}$ are time-diminishing real positive sequences satisfying the Robbins-Monro condition, i.e.,~$\sum_{t = 1}^\infty  {\beta _t^\theta }  = \infty$, $\sum_{t = 1}^\infty  {{{(\beta _t^\theta )}^2}}  < \infty$, $\sum_{t = 1}^\infty  {\beta _t^y}  = \infty$, and $\sum_{t = 1}^\infty  {{{(\beta _t^y)}^2}}  < \infty$,  then Algorithm~\ref{alg:mvp} will converge such that $\;\lim_{t\rightarrow \infty}\mathbb{E}[\|\nabla\hat{f}_{\lambda}(x_t)  \|_2] = 0$.
\end{theorem}
The proof of Theorem~\ref{thm:asyn} follows from the analysis in~\citet{xu2013block}. Due to space constraint, we report it in Appendix~\ref{sec:theory_app}. 

Algorithm~\ref{alg:mvp} is a special case of nonconvex block stochastic gradient (BSG) methods. To the best of our knowledge, no finite-sample analysis has been reported for this class of algorithms. Motivated by the recent papers by~\citet{nemirovski2009robust},~\citet{ghadimi2013stochastic},~\citet{cd:bsg:xu2015}, and~\citet{cd:sbmd:dang2015}, in Section~\ref{sec:theory}, we provide a finite-sample analysis for general nonconvex block stochastic gradient methods and apply it to Algorithm~\ref{alg:mvp} with Option II.

%%%%%%%%%%%%%%%%%%%%%%%%%%%%%%
%%%%%%%%%%%%%%%%%%%%%%%%%%%%%%
%%%%%%%%%%%%%%%%%%%%%%%%%%%%%%

\section{Finite-Sample Analysis}
\label{sec:theory}
\vspace{-0.3cm}
In this section, we first present a finite-sample analysis for the general class of nonconvex BSG algorithms~\citep{xu2013block}, for which there are no established results, in Section~\ref{sec:fs-bsg}. We then use these results and prove a finite-sample bound for our MVP algorithm with Option II, that belongs to this class, in Section~\ref{sec:fs-mvp}. Due to space constraint, we report the detailed proofs in Appendix~\ref{sec:theory_app}.

%%%%%%%%%%%%%%%%%%%%%%%%%%%%%

\subsection{Finite-Sample Analysis of Nonconvex BSG Algorithms}
\label{sec:fs-bsg}

In this section, we provide a finite-sample  analysis of the general nonconvex block stochastic gradient (BSG) method, where the problem formulation is given by
\begin{align}
\min_{x \in \mathbb{R}^{n}} ~ f(x) = \mathbb{E}_{\xi}[F(x,\xi)].
\label{eq:obj}
\end{align}
$\xi$ is a random vector, and $F(\cdot,\xi): \mathbb{R}^{n} \rightarrow \mathbb{R}$ is continuously differentiable and possibly nonconvex for every $\xi$. 
% \notey{$\|\cdots\|_2$ already appears very early. Move the sentence about $\|\cdots\|_2$ to where it is first used.FIXED-BO}
%
The variable $x \in \mathbb{R}^{n}$ can be partitioned into $b$ disjoint blocks as $x = \{x^1,x^2,\dotsc,x^b\}$, where $x^i \in \mathbb{R}^{n_i}$ denotes the $i$-th block of variables, and $\sum_{i = 1}^{b}n_i = n$.  
% We use superscript of the variable to denote the block of variable, and subscript of the variable to denote the number of iterations. 
For simplicity, we use $x^{<i}$ for $(x_i,\dots,x_{i-1})$, and $x^{\leq i}$,$x^{>i}$, and $x^{\geq i}$ are defined correspondingly.
We also use $\nabla_{x^i}$ to denote $\frac{\partial}{\partial x^i}$  for the partial gradient with respect to $x^i$.  $\Xi_t$ is the sample set generated at $t$-th iteration, and $\mathbf \Xi_{[t]} = (\Xi_1,\dots,\Xi_t)$ denotes the history of sample sets from the first through $t$-th iteration. $\{\beta_t^i: i = 1,\cdots,b\}_{t = 1}^{\infty}$ are denoted as the stepsizes. Also, let $\beta_t^{\max} = \max_{i}\beta_t^i$, and $\beta_t^{\min} = \min_{i}\beta_t^i$.
Similar to Algorithm \ref{alg:mvp}, the BSG algorithm cyclically updates all blocks of variables in each iteration, and the detailed algorithm for BSG method is presented in Appendix \ref{sec:probsg}.

Without loss of generality, we assume a fixed update order in the BSG algorithm.
% $
% \pi_{t}^{i} = i,
% $ 
% for all $i$ and $k$.
Let $\Xi_t = \{\xi_{t,1},\dotsc,\xi_{t,m_t}\}$ be the samples in the $t$-th iteration with size $m_t \geq 1$. 
Therefore, the stochastic partial gradient is computed as
$
\tilde g_t^i = \frac{1}{m_t} \sum_{l = 1}^{m_t}\nabla_{x^i}F(x_{t+1}^{<i}, x_t^{\geq i}; \xi_{t,l}).
$
Similar to Section \ref{sec:alg}, we define $g_t^i = \nabla_{x^i}f(x_{t+1}^{<i}, x_t^{\geq i})$, and the approximation error as $\Delta_t^i = \tilde g_t^i - g_t^i$.
% \begin{align}
% g_t^i = \nabla_{x^i}f(x_{t+1}^{<i}, x_t^{\geq i}),\quad 
% \Delta_t^i = \tilde g_t^i - g_t^i.
% \end{align}
%
We assume that the objective function $f$ is bounded and Lipschitz smooth, i.e., there exists a positive Lipschitz constant $L > 0$ such that
$
\|\nabla_{x^i}f(x) - \nabla_{x^i}f(y)\|_2 \leq L \|x - y\|_2
$,
 $\forall i \in \{1,\dots,b\}$ and  $\forall x,y \in \mathbb R^n$. Each block gradient of $f$ is also bounded, i.e., there exist a positive constant $G$ such that
$
% \label{eq:bdgrad}
\|\nabla_{x^i}f(x)\|_2 \leq G
$,
for any $i \in \{1,\dots,b\}$ and any $x \in \mathbb R^n$.
We also need Assumption \ref{asm3} for all block variables, i.e., $\mathbb{E}[\|\Delta_t^i\|_2] \leq \sigma$, for any $i$ and $t$. Then we have the following lemma.
\begin{lemma}
\label{lem:err_bound}
For any $i$ and $t$, there exist a positive constant $A$, such that
\begin{align}
\|\mathbb{E}[\Delta_t^i | \mathbf \Xi_{[t - 1]}]\|_2 \leq A\beta_t^{\max}.
\label{eq:stocerrbd}
\end{align}
\end{lemma}
The proof of Lemma~\ref{lem:err_bound} is in Appendix~\ref{sec:probsg}.
It should be noted that in practice, it is natural to take the final iteration's result as the output as in Algorithm~\ref{alg:mvp}. However, a standard strategy for analyzing nonconvex optimization methods is to pick up one previous iteration's result randomly according to a discrete probability distribution over $\{1,2,\dotsc, N\}$~\citep{nemirovski2009robust,ghadimi2013stochastic,cd:sbmd:dang2015}.
Similarly, our finite-sample analysis is based on the strategy that randomly pick up $\bar x_{N} = x_{z}$ according to
\begin{align}
\Pr(z = t) = \frac{\beta_t^{\min} - \frac{ L }{2}(\beta_t^{\max})^2}{\sum_{t = 1}^N (\beta_t^{\min} - \frac{ L }{2}(\beta_t^{\max})^2)}, \ t = 1,\dotsc, N.
\label{eq:prob-z}
\end{align}
Now we provide the finite-sample analysis result for the general nonconvex BSG algorithm as in~\citep{cd:bsg:xu2015}.
\begin{theorem}
\label{thm:cyclicbcd}
Let the output of the nonconvex BSG algorithm be $\bar{x}_N = x_z$ according to Eq.~\eqref{eq:prob-z}. If stepsizes satisfy $2\beta_t^{\min} > L(\beta_t^{\max})^2$ for $t = 1,\cdots,N$, then we have 
% \notey{Should $x_R$ be $\bar{x}_N$ in the following result?FIXED-BO}
\begin{align}
&\mathbb{E}\left[\|\nabla f(\bar{x}_N)\|_2^2\right] \leq \frac{ f(x_1) - f^* + \sum_{t = 1}^N (\beta_t^{\max})^2 C_t}{\sum_{t = 1}^N (\beta_t^{\min} - \frac{ L }{2}(\beta_t^{\max})^2)},
\label{eq:cyclicbcd}
\end{align}
where $f^* = \max_{x} f(x)$. 
%\begin{align}
%
$
C_t =  (1 - \frac{ L }{2}\beta_t^{\max}) \sum_{i = 1}^{b}  L  \sqrt{\sum_{j < i} (G^2 + \sigma^2) } + b\left( A G + \frac{L}{2} \sigma^2 \right),
$
% \label{eq:Ck}
% \end{align}
where $G$ is the gradient bound, $L$ is the Lipschitz constants, $\sigma$ is the variance bound, and $A$ is defined in Eq.~\eqref{eq:stocerrbd}.
\end{theorem}
As a special case, we discuss the convergence rate with constant stepsizes $\mathcal O(1/\sqrt{N})$ in Corollary~\ref{cor:cyclicrate}, which implies that the sample complexity $N = \mathcal O(1 / \varepsilon^2)$ in order to find $\varepsilon$-stationary solution of problem \eqref{eq:obj}.
\begin{corollary}
\label{cor:cyclicrate}
If we take constant stepsize such that $\beta_t^i = \beta^i = \mathcal O(1/\sqrt{N})$ for any $t$, and let $\beta^{\max} \coloneqq \max_i \beta^i$, $\beta^{\min} \coloneqq \min_i \beta^i$, then we have 
$
\mathbb{E}\left[\|\nabla f(\bar{x}_N)\|_2^2\right] \leq \mathcal O \left( \sqrt{\frac{f(x_1) - f^* + C}{N}} \right),
$
where $C_t$ in Eq.~\eqref{eq:cyclicbcd} reduces to a constant $C$ defined as
$
C =  (1 - \frac{ L }{2}\beta^{\max}) \sum_{i = 1}^{b} L \sqrt{\sum_{j < i} (G^2 + \sigma^2) } + b \left( A G + \frac{L}{2} \sigma^2 \right).
$
\end{corollary}

%%%%%%%%%%%%%%%%%%%%%%%%%

\subsection{Finite-Sample Analysis of Algorithm~\ref{alg:mvp}}
\label{sec:fs-mvp}

We present the major theoretical results of this paper, i.e., the finite-sample analysis of Algorithm~\ref{alg:mvp} with Option II. The proof of Theorem \ref{thm:cyclicbcd2} is in Appendix~\ref{sec:theory_app}.
\begin{theorem}
\label{thm:cyclicbcd2}
Let the output of the Algorithm~\ref{alg:mvp} be $\bar{x}_N$ as in Theorem~\ref{thm:cyclicbcd}. 
% \notey{Refer the previous definition of $\bar{x}_N$FIXED-BO} 
If $\{\beta_t^{\theta}\}$, $\{\beta_t^{y}\}$ are constants as in Option II in Algorithm~\ref{alg:mvp}, and also satisfies $2\beta_t^{\min} > L(\beta_t^{\max})^2$ for $t = 1,\cdots,N$, we have 
% \begin{align}
% &\mathbb{E}[\|\nabla\hat{f}_\lambda(x_R) \|_2^2] \leq \dfrac{\hat{f}_\lambda^* - \hat{f}_\lambda(x_1) + \displaystyle\sum_{t = 1}^N (\beta_t^{\max})^2 C_t}{\displaystyle\sum_{t = 1}^N (\beta_t^{\max} - \dfrac{L}{2}(\beta_t^{\max})^2)}
% \end{align}
% \begin{align}
% &\mathbb{E}[\|\nabla\hat{f}_\lambda(x_z) \|_2^2] \leq \dfrac{\hat{f}_\lambda^* - \hat{f}_\lambda(x_1) + N (\beta_t^{\max})^2 C}{N (\beta_t^{\min} - \dfrac{L}{2}(\beta_t^{\max})^2)}
% \label{eq:thm2}
% \end{align}
\begin{align}
&\mathbb{E}\left[\|\nabla \hat f_{\lambda}(\bar{x}_N)\|_2^2\right] \leq \dfrac{\hat{f}_\lambda^* - \hat{f}_\lambda(x_1) + N (\beta_t^{\max})^2 C}{N (\beta_t^{\min} - \frac{L}{2}(\beta_t^{\max})^2)}
\label{eq:thm2}
\end{align}
where $\hat{f}_{\lambda}^* = \max_{x} \hat f_\lambda (x)$, and
\begin{align}
% C = A M_\rho + L\sigma^2 + 2L(1 + L\beta_t^{\max})(3 \sigma^2 + 2 M_\rho^2),
% C = & (2L - L^2\beta^{\max}) \sqrt{2G^2 + \sigma^2 } + \left( 2A G + L \sigma^2 \right).
C = & (1 - \dfrac{L}{2}\beta_t^{\max})(L^2 \beta_{t}^{\max}(G^2 + \sigma^2) + L(2G^2 + \sigma^2)) + A G + L\sigma^2 + 2L(1 + L\beta_t^{\max})(3 \sigma^2 + 2 G^2).  
\label{eq:C}
\end{align}
% \notet{$C$ is to be modified, will fix it soon. Hi, TY, have you figured this out? FIXED}
% \notey{Also relate $\beta_t^{\min}$ and $\beta_t^{\max}$ to $\beta_t^\theta$ and $\beta_t^y$, and mention the condition $2\beta_t^{\min}>L \beta_t^{\max}$.TY, we leave this to you.} \notet{FIXED} 
% A is defined in Assumption \ref{asm:var}, and $G$ is the gradient bounds of $\|\nabla_y \hat f_{\lambda}(x)\|_2$, $\|\nabla_{\theta} \hat f_{\lambda}(x)\|_2$ for any $x$. \notet{changed to gradient bound here}
\end{theorem}
\vspace{-0.5cm}
\begin{proof}[Proof Sketch]
The proof follows the following major steps.

\textbf{(\RomanNum{1})}. First, we need to prove the bound of each block coordinate gradient, i.e., $\mathbb{E}[\|g_t^\theta\|_2^2]$ and $\mathbb{E}[ \|g_t^y\|_2^2]$, which is bounded as
\begin{align}
&(\beta_t^{\min} - \dfrac{L}{2}(\beta_t^{\max})^2)\mathbb{E}[\|g_t^\theta\|_2^2 + \|g_t^y\|_2^2]  \\
\leq&  \mathbb{E}[\hat{f}_\lambda(x_{t + 1})] - \mathbb{E}[\hat{f}_\lambda(x_t)] + (\beta_t^{\max})^2 A M_\rho + L(\beta_t^{\max})^2\sigma^2  
 + 2L\beta_t^{\max}(\beta_t^{\max} + L(\beta_t^{\max})^2)(3 \sigma^2 + 2 G^2).
\end{align}
Summing up over $t$, we have
\begin{align}
&\sum_{t = 1}^N (\beta_t^{\min} - \dfrac{L}{2}(\beta_t^{\max})^2)\mathbb{E}[\|g_t^\theta\|_2^2 + \|g_t^y\|_2^2]  \\
\leq&  \hat{f}_\lambda^* - \hat{f}_\lambda(x_1) +\sum_{t = 1}^N [ (\beta_t^{\max})^2 A G + L(\beta_t^{\max})^2\sigma^2  + 2L\beta_t^{\max}(\beta_t^{\max} + L(\beta_t^{\max})^2)(3 \sigma^2 + 2 G^2)].
\end{align}
% \noteb{Mohammad, shall we reference this equation number in the Appendix or not?}
%
\textbf{(\RomanNum{2})}. Next, we need to bound $\mathbb{E}[\|\nabla\hat{f}_\lambda(x_t) \|_2^2]$ using $\mathbb{E}[\|g_t^\theta\|_2^2 + \|g_t^y\|_2^2]$, which is proven to be
\begin{align}
& \mathbb{E}[\|\nabla\hat{f}_\lambda(x_t) \|_2^2] \leq L^2 (\beta_{t}^{\max})^2(G^2 + \sigma^2) + L\beta_{t}^{\max}(2G^2 + \sigma^2) + \mathbb{E}[\|g_t^\theta\|_2^2 + \|g_t^y\|_2^2].
\end{align}
\textbf{(\RomanNum{3})}. Finally, combining (\RomanNum{1}) and (\RomanNum{2}), and rearranging the terms, Eq.~\eqref{eq:thm2} can be obtained as a special case of Theorem~\ref{thm:cyclicbcd}, which completes the proof. 
\end{proof}

%%%%%%%%%%%%%%%%%%%%%%%%%%%%%

%%%%%%%%%%%%%%%%%%%%%%%%%%%%%
%%%%%%%%%%%%%%%%%%%%%%%%%%%%%
%%%%%%%%%%%%%%%%%%%%%%%%%%%%%

\section{Experimental Study}
\label{sec:experimental}

In this section, we evaluate our MVP algorithm with Option I in three risk-sensitive domains: the portfolio management~\citep{saferl:castro2012}, the American-style option~\citep{tamar2014scaling}, and the optimal stopping~\citep{saferl:cvar:chow2014,Chow18RC}. The baseline algorithms are the vanilla policy gradient (PG), the mean-variance policy gradient in~\citet{saferl:castro2012}, the stochastic gradient ascent (SGA) applied to our optimization problem~\eqref{eq:max-alt}, and the randomized coordinate ascent policy gradient (RCPG), i.e.,~the SBMD-based version of our algorithm. Details of SGA and RCPG can be found in Appendix~\ref{sec:rcpg-sga}. For each algorithm, we optimize its Lagrangian parameter $\lambda$ by grid search and report the mean and variance of its return random variable as a Gaussian.\footnote{Note that the return random variables are not necessarily Gaussian, we only use Gaussian for presentation purposes.} Since the algorithms presented in the paper (MVP and RCPG) are policy gradient, we only compare them with Monte-Carlo based policy gradient algorithms and do not use any actor-critic algorithms, such as those in~\citet{saferl:prashanth2013} and TRPO~\citep{trpo:schulman:2015}, in the experiments. 

\begin{figure}[t]
\begin{subfigure}{0.32\linewidth}
\includegraphics[width=\linewidth]{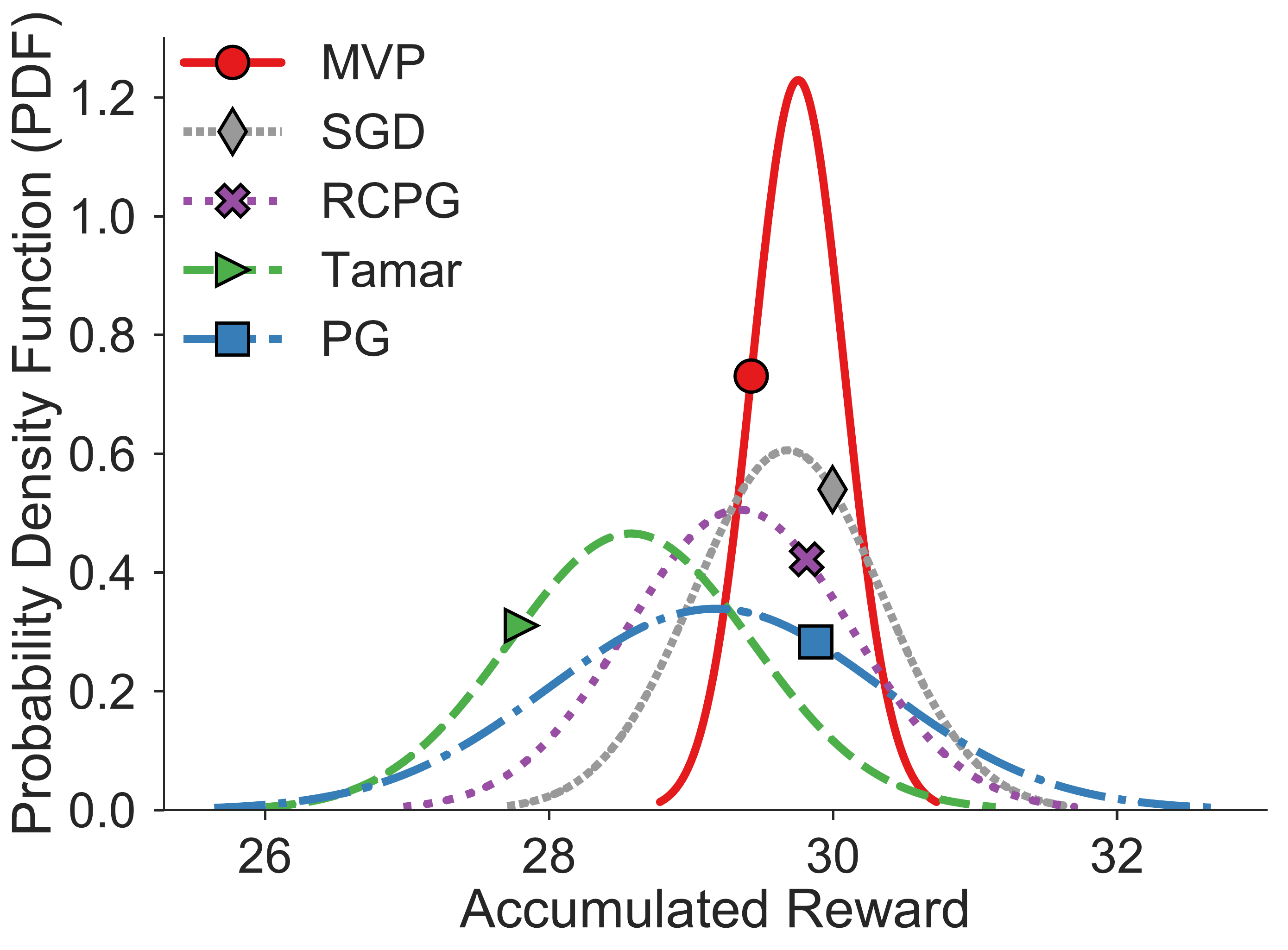}
\caption{Portfolio management domain}
\label{fig:PF}
\end{subfigure}
\begin{subfigure}{0.32\linewidth}
\includegraphics[width=\linewidth]{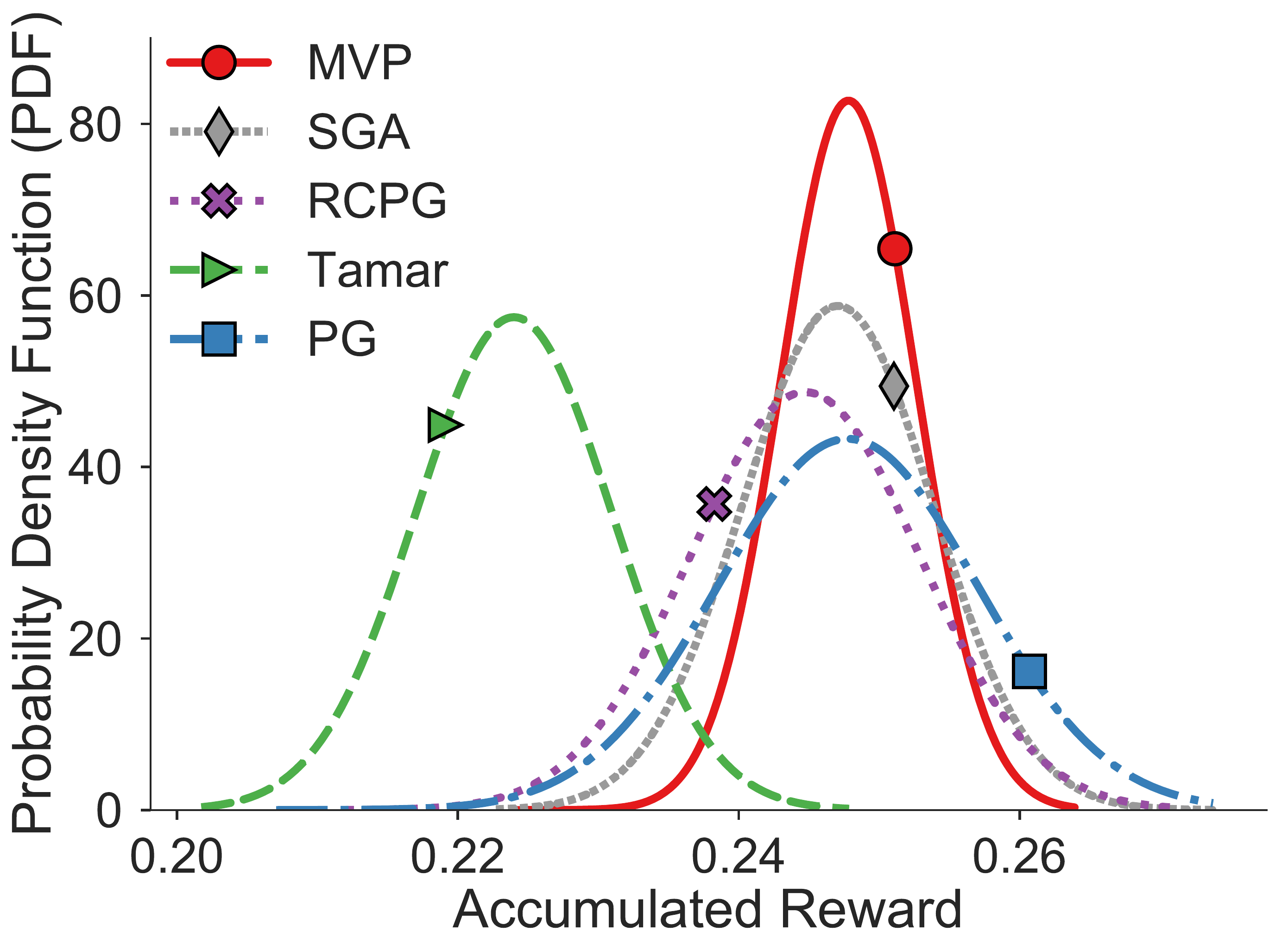}
\caption{American-style option domain}
\label{fig:AO}
\end{subfigure}
\begin{subfigure}{0.32\linewidth}
\includegraphics[width=\linewidth]{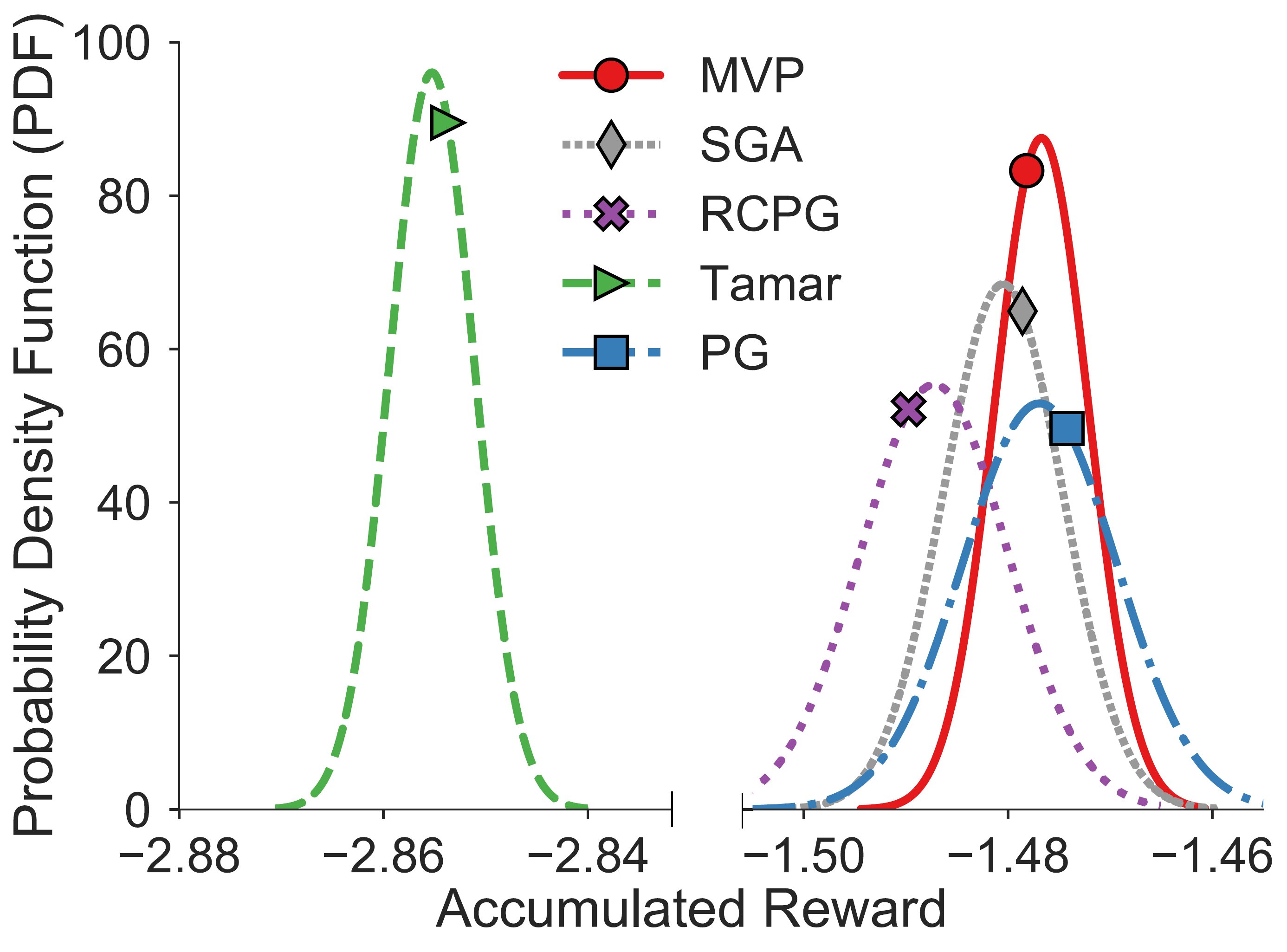}
\caption{Optimal stopping domain}
\label{fig:OS}
\end{subfigure}
\caption{Empirical results of the distributions of the return (cumulative rewards) random variable. Note that markers only indicate different methods.}
\end{figure}

\subsection{Portfolio Management}
%
% \noteb{1. note that we use Gaussian here, but in reality nothing is Gaussian. 2. remove subfigure (b)}
The portfolio domain~\cite{saferl:castro2012} is composed of the liquid and non-liquid assets. A liquid asset has a fixed interest rate $r_l$ and can be sold at any time-step $k \le {\tau}$. A non-liquid asset can be sold only after a fixed period of $W$ time-steps with a time-dependent interest rate $r_{\text{nl}}(k)$, which can take either $r_{\text{nl}}^{\text{low}}$ or $r_{\text{nl}}^{\text{high}}$, and the transition follows a switching probability $p_{\text{switch}}$. The non-liquid asset also suffers a default risk (i.e.,~not being paid) with a probability $p_{\text{risk}}$. All investments are in liquid assets at the initial time-step $k = 0$. At the $k$-th step, the state is denoted by $x(k)\in \mathbb{R}^{W+2}$, where $x_1 \in [0,1]$ is the portion of the investment in liquid assets, $x_2,\cdots,x_{W+1} \in [0,1]$ is the portion in non-liquid assets with time to maturity of $1,\cdots, W$ time-steps, respectively, and $x_{W+2}(k) = r_{\text{nl}}(k)-\mathbb{E}[r_{\text{nl}}(k)]$. 
%as Figure~\ref{fig:dynamic}  shows. 
The investor can choose to invest a fixed portion $\eta$ ($0<\eta<1$) of his total available cash in the non-liquid asset or do nothing.
% if the investor has at least $\eta$ portion in the liquid asset, otherwise he has to wait for the maturity of the non-liquid asset. 
% The reward at the $k$-th step is defined as the profit from the investment. 
More details about this domain can be found in~\citet{saferl:castro2012}.
%  $r_l = 0.01$, $r_{\text{nl}}^{\text{high}} = 2$, $r_{\text{nl}}^{\text{low}} = 1.1$, $p_{\text{risk}} = 0.05$, $p_{\text{switch}} = 0.1$, $W = 4$.
%
Figure~\ref{fig:PF} shows the results of the algorithms. PG has a large variance and the Tamar's method has the lowest mean return. The results indicate that MVP yields a higher mean return with less variance compared to the competing algorithms.

\subsection{American-style Option}

An American-style option~\cite{tamar2014scaling} is a contract that gives the buyer the right to buy or sell the asset at a strike price $W$ at or before the maturity time ${\tau}$. The initial price of the option is $x_{0}$, and the buyer has bought a put option with the strike price $W_{\text{put}}<x_{0}$ and a call option with the strike price $W_{\text{call}}>x_{0}$. 
% $x_{k+1}  = \begin{cases}
%    f_u x_k, \mathrm{with\ prob.}\  p,\\
%    f_d x_k, \mathrm{with\ prob.}\  1-p,
% \end{cases}$
% %
%
At the $k$-th step ($k \le \tau$), the state is $\{x_{k},k\}$, where $x_k$ is the current price of the option. The action $a_{k}$ is either executing the option or holding it. $x_{k+1}$ is $f_u x_k$ w.p.~$p$ and $f_d x_k$ w.p.~$1 - p$, where $f_u$ and $f_d$ are constants. The reward is $0$ unless an option is executed and the reward for executing an option is ${r_k} = \max (0,{W_{{\rm{put}}}} - {x_k}) + \max (0,{x_k} - {W_{{\rm{call}}}})$. More details about this domain can be found in~\citet{tamar2014scaling}.
% Once an option is executed, or $k = {\tau}$,
% a transition to terminal state occurs and the reward is given. 
% $x_0 = 1.25$, $f_u = 9 / 8$, $f_d = 8 / 9$, $p = 0.45$, and $\tau = 20$.
%
Figure~\ref{fig:AO} shows the performance of the algorithms. The results suggest that MVP can yield a higher mean return with less variance compared to the other algorithms. 
% In addition, MVP appears to have a better chance to avoid getting zero reward, which suggests that MVP performs a better risk-sensitive property.

\subsection{Optimal Stopping}
The optimal stopping problem~\citep{saferl:cvar:chow2014,Chow18RC} is a continuous state domain. At the $k$-th time-step ($k \le \tau$, $\tau$ is the stopping time), the state is $\{x_k, k\}$, where $x_k$ is the cost. The buyer decide either to accept the present cost or wait. If the buyer accepts or when $k = T$, the system reaches a terminal state and the cost $x_k$ is received,
% \notet{the system reaches a terminal state and the cost $\max(x_k,K)$ is received, where $K$ is the maximum cost threshold} 
otherwise, the buyer receives the cost $p_h$ and the new state is $\{x_{k+1}, k+1\}$, where $x_{k+1}$ is $f_{u}x_{k}$ w.p.~$p$ and $f_{d}x_{k}$ w.p. $1-p$ ($f_{u} > 1$ and $f_{d}<1$ are constants). More details about this domain can be found in~\citet{saferl:cvar:chow2014}.
Figure~\ref{fig:OS} shows the performance of the algorithms. The results indicate that MVP is able to yield much less variance without affecting its mean return. 
We also summarize the performance of these algorithms in all three risk-sensitive domains as Table \ref{tb:meanvartb}, where Std is short for Standard Deviation. 

\begin{table}[htb]
\centering
\begin{tabular}{|c|c|c|c|c|c|c|}
\hline
\multirow{2}{*}{\textbf{}} & \multicolumn{2}{c|}{\textbf{Portfolio Management}} & \multicolumn{2}{c|}{\textbf{American-style Option}} & \multicolumn{2}{c|}{\textbf{Optimal Stopping}} \\ \cline{2-7} 
                           & \textbf{Mean}      & \textbf{Std}   & \textbf{Mean}      & \textbf{Std}    & \textbf{Mean}    & \textbf{Std} \\ \hline
\textbf{MVP}               & \textbf{29.754}    & \textbf{0.325}                & \textbf{0.2478}    & \textbf{0.00482}               & \textbf{-1.4767} & 0.00456                     \\ \hline
\textbf{PG}                & 29.170             & 1.177                         & 0.2477             & 0.00922                        & -1.4769          & 0.00754                     \\ \hline
\textbf{Tamar}             & 28.575             & 0.857                         & 0.2240             & 0.00694                        & -2.8553          & \textbf{0.00415}            \\ \hline
\textbf{SGA}               & 29.679             & 0.658                         & 0.2470             & 0.00679                        & -1.4805          & 0.00583                     \\ \hline
\textbf{RCPG}              & 29.340             & 0.789                         & 0.2447             & 0.00819                        & -1.4872          & 0.00721                     \\ \hline
\end{tabular}
\vspace{\baselineskip}
\caption{Performance Comparison among Algorithms}
\label{tb:meanvartb}
\end{table}

%% depend on space
% In all, of all the domains, PG usually has the highest cumulative rewards with a large variance indicating a high risk, and Tamar's method controls the risk at the cost of a much lower mean. The methods (MVP, RCPG, SGA) based on the coordinate ascent formulation in Eq.~\eqref{eq:max-alt} tend to have less loss of the cumulative rewards, and MVP performs the best among these approaches with relatively higher mean and less variance.

%%%%%%%%%%%%%%%%%%%%%%%%%%%%%
%%%%%%%%%%%%%%%%%%%%%%%%%%%%%
%%%%%%%%%%%%%%%%%%%%%%%%%%%%%
\section{Conclusion}
\label{sec:conclusion}
\vspace{-3mm}
This paper is motivated to provide a risk-sensitive policy search algorithm with provable sample complexity analysis to maximize the mean-variance objective function. To this end, the objective function is reformulated based on the Legendre-Fenchel duality, and a novel stochastic block coordinate ascent algorithm is proposed with in-depth analysis. 
There are many interesting future directions on this research topic. Besides stochastic policy gradient, deterministic policy gradient \citep{dpg:silver2014} has shown great potential in large discrete action space. It is interesting to design a risk-sensitive deterministic policy gradient method. 
Secondly, other reformulations of the mean-variance objective function are also worth exploring, which will lead to new families of algorithms. 
Thirdly, distributional RL~\citep{distributional2016} is strongly related to risk-sensitive policy search, and it is interesting to investigate the connections between risk-sensitive policy gradient methods and distributional RL.
Last but not least, it is interesting to test the performance of the proposed algorithms together with other risk-sensitive RL algorithms on highly-complex 
 risk-sensitive tasks, such as autonomous driving problems and other challenging tasks.

\section*{Acknowledgments}

Bo Liu, Daoming Lyu, and Daesub Yoon were partially supported by a grant (18TLRP-B131486-02) from Transportation and Logistics R\&D Program funded by Ministry of Land, Infrastructure and Transport of Korean government. Yangyang Xu was partially supported by the NSF grant DMS-1719549.

%%%%%%%%%%%%%%%%%%%%%%%%%%%%%
%%%%%%%%%%%%%%%%%%%%%%%%%%%%%
%%%%%%%%%%%%%%%%%%%%%%%%%%%%%
% \clearpage
\bibliographystyle{apalike}
\bibliography{ref}

%%%%%%%%%%%%%%%%%%%%%%%%%%%%%
%%%%%%%%%%%%%%%%%%%%%%%%%%%%%
%%%%%%%%%%%%%%%%%%%%%%%%%%%%%
\clearpage
\begin{center}
{\Large \textbf{Appendix}}
\end{center}

\appendix
\input{appendix}
\end{document}

%% file: appendix.tex
%%%%%%%%%%%%%%%%%%%%%%%%%%%
%%%%%%%%%%%%%%%%%%%%%%%%%%%
%%%%%%%%%%%%%%%%%%%%%%%%%%%
\section{Theoretical Analysis of Algorithm~\ref{alg:mvp}}
\label{sec:theory_app}
Now we present the theoretical analysis of Algorithm~\ref{alg:mvp}, with both asymptotic convergence and finite-sample error bound analysis.
For the purpose of clarity, in the following analysis, $x$ is defined as $x := [y,\theta]^\top$, where  $x^1 :=y,\ x^2 :=\theta$. Similarly, $g_t^i$ (resp. $\beta^i_t$) is used to denote $g_t^y$(resp. $\beta^y_t$) and $g_t^\theta$(resp. $\beta^\theta_t$), where  $g_t^1 :=g_t^y,\ g_t^2 :=g_t^\theta$(resp. $\beta^1_t :=\beta^y_t,\ \beta^2_t :=\beta^\theta_t$). Let $\beta_{t}^{\max} = \max\{\beta_{t}^1,\beta_{t}^2\}$, $\beta_{t}^{\min} = \min\{\beta_{t}^1,\beta_{t}^2\}$, and $\|\cdot\|_2$ denotes the Euclidean norm.
\subsection{Asymptotic Convergence Proof Algorithm~\ref{alg:mvp}}
\label{sec:asym}

In this section, we provide the asymptotic convergence analysis for Algorithm~\ref{alg:mvp} with Option I, i.e., the stepsizes are chosen to satisfy the Robbins-Monro condition and the output is the last iteration's result. We first introduce a useful lemma, which follows from Lemma \ref{lem:err_bound}.

\begin{lemma}
\label{lem:prod}
Under Assumption \ref{asm3}, we have
\begin{align}
\mathbb{E}[\langle \nabla_{y}\hat{f}_\lambda (x_t), \Delta_t^y \rangle] = & 0\\
\mathbb{E}[\langle \nabla_{\theta}\hat{f}_\lambda (x_t), \Delta_t^\theta \rangle] \leq & \beta^{\max}_t A G,
\end{align}
where $A$ is defined in Eq~\eqref{eq:stocerrbd}.
\end{lemma}
\begin{proof}
Let $\Xi_t$ denote the history of random samples from the first to the $t$--th episode, then $x_t$ is independent of $\Delta_t^i$ conditioned on $\Xi_{t - 1}$ (since $x_t$ is deterministic conditioned on $\Xi_{t - 1}$). Then we have
\begin{align}
&\mathbb{E}[\langle \nabla_{y}\hat{f}_\lambda (x_t), \Delta_t^y \rangle] = \mathbb{E}_{\Xi_{t - 1}}[\mathbb{E}[\langle \nabla_{y}\hat{f}_\lambda (x_t), \Delta_t^y \rangle|\Xi_{t - 1}]]  \\
= & \mathbb{E}_{\Xi_{t - 1}}[\langle \mathbb{E}[\nabla_{y}\hat{f}_\lambda (x_t)|\Xi_{t - 1}], \mathbb{E}[\Delta_t^y|\Xi_{t - 1}] \rangle] = 0,
\label{eq:expbound1}
\end{align}
\begin{align}
&\mathbb{E}[\langle \nabla_{\theta}\hat{f}_\lambda (x_t), \Delta_t^\theta \rangle] = \mathbb{E}_{\Xi_{t - 1}}[\mathbb{E}[\langle \nabla_{\theta}\hat{f}_\lambda (x_t), \Delta_t^\theta \rangle|\Xi_{t - 1}]]  \\
= & \mathbb{E}_{\Xi_{t - 1}}[\langle \mathbb{E}[\nabla_{\theta}\hat{f}_\lambda (x_t)|\Xi_{t - 1}], \mathbb{E}[\Delta_t^\theta|\Xi_{t - 1}] \rangle]   \\
\leq & \mathbb{E}_{\Xi_{t - 1}}[\| \mathbb{E}[\nabla_{\theta}\hat{f}_\lambda (x_t)|\Xi_{t - 1}]\|_2 \cdot \|\mathbb{E}[\Delta_t^\theta|\Xi_{t - 1}] \|_2]  \\
\leq & \beta^{\max}_t A \mathbb{E}_{\Xi_{t - 1}}[\| \mathbb{E}[\nabla_{\theta}\hat{f}_\lambda (x_t)|\Xi_{t - 1}]\|_2]   \\
\leq & \beta^{\max}_t A \mathbb{E}[\|\nabla_{\theta}\hat{f}_\lambda (x_t)\|_2] \leq \beta^{\max}_t A G,
\label{eq:expbound2}
\end{align}
% \note[lb]{please write it as a Lemma, put it to the Appendix. Also point out it is similar to which lemma in Dr. Xu's paper.}
where the second equality in both Eq.\eqref{eq:expbound1} and Eq.\eqref{eq:expbound2} follows from the conditional independence between $x_t$ and $\Delta_t^i$, the second inequality in Eq.\eqref{eq:expbound2} follows from Assumption \ref{asm:var}, and the last inequality in Eq.\eqref{eq:expbound2} follows from the Jensen's inequality and $G$ is the gradient bound.
\end{proof}

Then, we introduce Lemma~\ref{lem:1}, which is essential for proving Theorem~\ref{thm:asyn}.
\begin{lemma}
\label{lem:1}
For two non-negative scalar sequences $\{a_t\}$ and $\{b_t\}$, if $\sum_{t = 1}^{\infty}a_t = +\infty$ and $\sum_{t = 1}^{\infty}a_t b_t < +\infty$, we then have
\begin{equation} 
\mathop {\lim }_{t \to \infty } \inf {b_t} = 0.
\end{equation}
Further, if there exists a constant $K>0$ such that $|b_{t + 1} - b_{t}| \leq a_t K$, then
\begin{equation}
\lim_{t \rightarrow \infty}b_t = 0.
\end{equation}
\end{lemma} 
The detailed proof can be found in Lemma A.5 of \citep{mairal2013stochastic} and Proposition 1.2.4 of \citep{bertsekas:npbook}. Now it is ready to prove Theorem~\ref{thm:asyn}.

\begin{proof}[Proof of Theorem \ref{thm:asyn}]

We define $\Gamma_t^1 := \hat{f}(x_{t}^{1},x_t^{2}) - \hat{f}(x_{t + 1}^{1},x_t^{2}), \Gamma_t^2 := \hat{f}(x_{t + 1}^{1},x_t^{2}) - \hat{f}(x_{t + 1}^{1},x_{t + 1}^{2})$ to denote the block update. It turns out that for $i=1,2$, and $\Gamma_t^i$ can be bounded following the Lipschitz smoothness as
\begin{align}
\Gamma_t^i
\leq & \langle g_t^i, x_t^i - x_{t+1}^i \rangle + \dfrac{L}{2}\|x_{t}^i - x_{t + 1}^i \|_2^2   \\
= & -\beta_t^i \langle g_t^i, \Tilde{g}_t^i \rangle + \dfrac{L}{2}(\beta_t^i)^2\|\Tilde{g}_t^i\|_2^2  \\
= & -(\beta_t^i - \dfrac{L}{2}(\beta_t^i)^2)\|g_t^i\|_2^2 + \dfrac{L}{2}(\beta_t^i)^2)\|\Delta_t^i\|_2^2 - (\beta_t^i - L(\beta_t^i)^2)\langle g_t^i,\Delta_t^i \rangle   \\
= & -(\beta_t^i - \dfrac{L}{2}(\beta_t^i)^2)\|g_t^i\|_2^2 + \dfrac{L}{2}(\beta_t^i)^2\|\Delta_t^i\|_2^2 \\
& - (\beta_t^i - L(\beta_t^i)^2) (\langle g_t^i - \nabla_{x^i}\hat{f}_\lambda (x_t), \Delta_t^i \rangle +  \langle \nabla_{x^i}\hat{f}_\lambda (x_t), \Delta_t^i \rangle)
\label{eq:lipthmq0}
\end{align}
where the equalities follow the definition of $\Delta_t^i$ and the update law of Algorithm~\ref{alg:mvp}.
We also have the following argument
\begin{align}
& - (\beta_t^i - L(\beta_t^i)^2)\langle g_t^i - \nabla_{x^i}\hat{f}_\lambda (x_t), \Delta_t^i \rangle  \\
\leq & |\beta_t^i - L(\beta_t^i)^2| \|\Delta_t^i\|_2 \|g_t^i - \nabla_{x^i}\hat{f}_\lambda (x_t)\|_2  \\
% \leq & L |\beta_t^i - L(\beta_t^i)^2| \|\Delta_t^i\|_2 \|x_{t + 1}^{<i}  - x_{t}^{<i}\|_2  \\
\leq & L |\beta_t^i - L(\beta_t^i)^2| \|\Delta_t^i\|_2 \|x_{t + 1}  - x_{t}\|_2  \\
\leq & L |\beta_t^i - L(\beta_t^i)^2| \|\Delta_t^i\|_2 \sqrt{\sum_{j = 1}^2 \|\beta_t^j \tilde{g}_t^j\|_2^2}  \\
\leq & L (\beta_t^i + L(\beta_t^i)^2) \beta_t^{\max} \left(\|\Delta_t^i\|_2 + \sum_{j = 1}^2 (\|g_t^j\|_2^2 + \|\Delta_t^j\|_2^2) \right),
\label{eq:lipargu}
\end{align}
where the first inequality follows from Cauchy-Schwarz inequality, the second inequality follows from the Lipschitz smoothness of objective function $\hat{f}_\lambda$, the third inequality follows from the update law of Algorithm~\ref{alg:mvp}, and the last inequality follows from the triangle inequality.
Combining Eq.~\eqref{eq:lipthmq0} and Eq.~\eqref{eq:lipargu}, we obtain
\begin{align}
\Gamma_t^i
\leq & -(\beta_t^i - \dfrac{L}{2}(\beta_t^i)^2)\|g_t^i\|_2^2 + \dfrac{L}{2}(\beta_t^i)^2\|\Delta_t^i\|_2^2   \\
& - (\beta_t^i - L(\beta_t^i)^2)\langle \nabla_{x^i}\hat{f}_\lambda (x_t), \Delta_t^i \rangle   \\
& + L(\beta_t^i + L(\beta_t^i)^2) \cdot \beta_t^{\max} \left(\|\Delta_t^i\|_2^2 + \sum_{j = 1}^2(\|g_t^j\|_2^2 + \|\Delta_t^j\|_2^2)\right).
\label{eq:lipthmq}
\end{align}
% \note[lb]{TY, remember to change places such as $\max_j \beta^j_t$ to $\beta_t^{\max}$}

Summing Eq.~\eqref{eq:lipthmq} over $i$, then we obtain
\begin{align}
\label{eq:summing}
&\hat{f}_\lambda(x_t) - \hat{f}_\lambda(x_{t + 1})   \\
\leq & -\sum_{i = 1}^2(\beta_t^i - \dfrac{L}{2}(\beta_t^i)^2)\|g_t^i\|_2^2 - \sum_{i = 1}^2(\beta_t^i - L(\beta_t^i)^2)\langle \nabla_{x^i}\hat{f}_\lambda (x_t), \Delta_t^i \rangle  \\
&+ \sum_{i = 1}^2 \left( \dfrac{L}{2}(\beta_t^i)^2\|\Delta_t^i\|_2^2  + L(\beta_t^i + L(\beta_t^i)^2)\beta_t^{\max}(\|\Delta_t^i\|_2^2 + \sum_{j = 1}^2(\|g_t^j\|_2^2 + \|\Delta_t^j\|_2^2)) \right).
\end{align}

We also have the following fact, 
\begin{align}
\mathbb{E}[\langle \nabla_{y}\hat{f}_\lambda (x_t), \Delta_t^y \rangle] = & 0\\
\mathbb{E}[\langle \nabla_{\theta}\hat{f}_\lambda (x_t), \Delta_t^\theta \rangle] \leq & \beta^{\max}_t A G.
\end{align}
We prove this fact in Lemma \ref{lem:prod} as a special case of Lemma \ref{lem:err_bound}, and the general analysis can be found in Lemma 1 in \citep{cd:bsg:xu2015}.
Taking expectation w.r.t. $t$ on both sides of the inequality Eq.~\eqref{eq:summing}, we have
\begin{align}
& \mathbb{E}[\hat{f}_\lambda(x_t)] - \mathbb{E}[\hat{f}_\lambda(x_{t + 1})]  \\
\leq & -\sum_{i = 1}^2(\beta_t^i - \dfrac{L}{2}(\beta_t^i)^2)\mathbb{E}[\|g_t^i\|_2^2] + (\beta_t^\theta - L(\beta_t^\theta)^2)\beta^{\max}_t A G   \\
& + \sum_{i = 1}^2\left( \dfrac{L}{2}(\beta_t^i)^2\mathbb{E}[\|\Delta_t^i\|_2^2]+ L(\beta_t^i + L(\beta_t^i)^2)\beta_t^{\max} (\mathbb{E}[\|\Delta_t^i\|_2^2] + \sum_{j = 1}^2(\mathbb{E}[\|g_t^j\|_2^2] + \mathbb{E}[\|\Delta_t^j\|_2^2])) \right)  \\
\leq & -\sum_{i = 1}^2(\beta_t^{\min} - \dfrac{L}{2}(\beta_t^{\max})^2)\mathbb{E}[\|g_t^i\|_2^2] + (\beta^{\max}_t)^2 A G  \\
& + ( L(\beta_t^{\max})^2\sigma^2 + 2L\beta_t^{\max}(\beta_t^{\max} + L(\beta_t^{\max})^2)(3 \sigma^2 + 2 G^2)),
\label{eq:exp}
\end{align}
where the first inequality follows from Eq.~\eqref{eq:expbound1} and Eq.~\eqref{eq:expbound2}, and the second inequality follows from the boundedness of $\mathbb{E}[\|\Delta_t^i\|_2]$ and $\mathbb{E}[\|g_t^i\|_2]$.
% $\mathbb{E}[\|\nabla_{x^i}\hat{f}_{\lambda}(x_{t + 1}^{<i}, x_{t}^{\geq i})\|_2]$ (i.e. the boundedness of $\mathbb{E}[\|g_t^i\|_2]$).

Rearranging Eq.~\eqref{eq:exp}, we obtain
\begin{align}
&\sum_{i = 1}^2(\beta_t^{\min} - \dfrac{L}{2}(\beta_t^{\max})^2)\mathbb{E}[\|g_t^i\|_2^2]  \\
\leq&  \mathbb{E}[\hat{f}_\lambda(x_{t + 1})] - \mathbb{E}[\hat{f}_\lambda(x_t)] + (\beta_t^{\max})^2 A M_\rho + L(\beta_t^{\max})^2\sigma^2  \\
& + 2L\beta_t^{\max}(\beta_t^{\max} + L(\beta_t^{\max})^2)(3 \sigma^2 + 2 G^2).
\label{eq:rearr}
\end{align}

By further assuming $0 < \inf_{t}\frac{\{\beta^{\theta}_t\}}{\{\beta^{y}_t\}} \leq \sup_{t}\frac{\{\beta^{\theta}_t\}}{\{\beta^{y}_t\}} < \infty$, it can be verified that $\{\beta^{\max}_t\}$ and $\{\beta^{\min}_t\}$ also satisfy Robbins-Monro condition. Note that $\hat{f}_{\lambda}$ is upper bounded, summing Eq.~\eqref{eq:rearr} over $t$ and using the Robbins-Monro condition of $\{\beta^{\theta}_t\}, \{\beta^{y}_t\}, \{\beta^{\max}_t\}, \{\beta^{\min}_t\}$,
we have
\begin{equation}
\label{eq:lemmf}
\sum_{t = 1}^{\infty}\beta_t^{\min}\mathbb{E}[\|g_t^i\|_2^2] < \infty, ~\forall i.
\end{equation}

Furthermore, let $\xi^1_t = (x_t^1, x_t^2)$ and $\xi^2_t = (x_{t + 1}^1, x_t^2)$, then
\begin{align}
|\mathbb{E}[\|g_{t + 1}^i\|_2^2] - \mathbb{E}[\|g_t^i\|_2^2]| \leq & \mathbb{E}[\|g_{t + 1}^i + g_{t}^i\|_2 \cdot \|g_{t + 1}^i - g_{t}^i\|_2]\\  
\leq & 2LM_{\rho}\mathbb{E}[\|\xi^i_{t + 1} - \xi^i_t\|_2] \\  
= & 2LM_{\rho}\mathbb{E}\left[\sqrt{\sum_{j < i}\|\beta_{t + 1}^j\Tilde{g}_{t + 1}^j\|_2^2 + \sum_{j \geq i}\|\beta_{t}^j\Tilde{g}_{t}^j\|_2^2}\right] \\  
\leq & 2LM_{\rho}\beta_{t}^{\max}\mathbb{E}\left[\sqrt{\sum_{j < i}\|\Tilde{g}_{t + 1}^j\|_2^2 + \sum_{j \geq i}\|\Tilde{g}_{t}^j\|_2^2} \right] \\  
\leq & 2LM_{\rho}\beta_{t}^{\max}\sqrt{\mathbb{E}[\sum_{j < i}\|\Tilde{g}_{t + 1}^j\|_2^2 + \sum_{j \geq i}\|\Tilde{g}_{t}^j\|_2^2]}\\ 
\leq & 2LM_{\rho}\beta_{t}^{\max}\sqrt{2(G^2 + \sigma^2)},
\label{eq:eq48}
\end{align}
where the first inequality follows from Jensen's inequality, the second inequality follows from the definition of gradient bound $G$ and the gradient Lipschitz continuity of $\hat{f}_{\lambda}$, the third inequality follows from the Robbins-Monro condition of $\{\beta_t^y\}$ and $\{\beta_t^\theta\}$, and the last two inequalities follow Jensen's inequality in probability theory.

Combining Eq.~\eqref{eq:lemmf} and Eq.~\eqref{eq:eq48} and according to Lemma \ref{lem:1}, we have $\lim_{t\rightarrow \infty}\mathbb{E}[\|g_t^i\|_2] = 0$ for $i = 1,2$ by Jensen's inequality. Hence, 
% \begin{align}
%  
% \mathbb{E}[\|\nabla_{x^i}\hat{f}_{\lambda}(x_t)\|_2] \leq & \mathbb{E}[\|\nabla_{x^i}\hat{f}_{\lambda}(x_t) - g_t^i\|_2] + \mathbb{E}[\|g_t^i\|_2] \\  
% \leq & L \cdot \mathbb{E}[\|x_{t + 1}^{<i} - x_t^{<i}\|_2] + \mathbb{E}[\|g_t^i\|_2]\\ 
% \leq & L\beta_t^{\max}\sqrt{4(M_{\rho}^2 + \sigma^2)} + \mathbb{E}[\|g_t^i\|_2]
% \end{align}
%
\begin{align}
\mathbb{E}[\|\nabla_{y}\hat{f}_{\lambda}(x_t)\|_2] =  & \mathbb{E}[\|g_t^y\|_2] \\
\mathbb{E}[\|\nabla_{\theta}\hat{f}_{\lambda}(x_t)\|_2] \leq & \mathbb{E}[\|\nabla_{\theta}\hat{f}_{\lambda}(x_t) - g_t^\theta\|_2] + \mathbb{E}[\|g_t^\theta\|_2]   \\
\leq & L \cdot \mathbb{E}[\|y_{t + 1} - y_t\|_2] + \mathbb{E}[\|g_t^\theta\|_2]   \\ 
\leq & L \beta_t^{\max} \mathbb{E}[\|\tilde{g}_{t}^{y}\|_2] + \mathbb{E}[\|g_t^\theta\|_2]   \\ 
\leq & L\beta_t^{\max}\sqrt{G^2 + \sigma^2} + \mathbb{E}[\|g_t^\theta\|_2]
\end{align}
where the first inequality follows from the triangle inequality, the second inequality follows from the Lipschitz continuity of $\hat{f}_{\lambda}$, and the last inequality follows from the same argument for Eq.~\eqref{eq:eq48}.
Also, note that $\lim_{t\rightarrow \infty}\beta_t^{\max} = 0$, $\lim_{t\rightarrow \infty}L\sqrt{G^2 + \sigma^2} < \infty$, and $\lim_{t\rightarrow \infty}\mathbb{E}[\|g_t^i\|_2] = 0$, so that when $t \to \infty $, $L\beta_t^{\max}\sqrt{G^2 + \sigma^2} + \mathbb{E}[\|g_t^y\|_2] + \mathbb{E}[\|g_t^\theta\|_2]\to 0$. This completes the proof.
\end{proof}
\begin{remark}
Different from the stringent two-time-scale setting where one stepsize needs to be ``quasi-stationary'' compared to the other~\citep{saferl:castro2012}, the stepsizes in Algorithm~\ref{alg:mvp} does not have such requirements, which makes it easy to tune in practice.
% In Theorem \ref{thm:asyn}, we proved the convergence of Algorithm~\ref{alg:mvp} asymptotically. This convergence has been established in terms of the Robbins-Monro condition for the single-time-scale stepsizes of the two blocks of the objective function $\hat{f}_\lambda$ (i.e. there does not has to be some relationship between stepsizes of the two blocks). \note[ty]{Is it okay now?} 
\end{remark}

%%%%%%%%%%%%%%%%%%%%%%%%%%%%%%%%
%%%%%%%%%%%%%%%%%%%%%%%%%%%%%%%%
\subsection{Finite-Sample Analysis of Algorithm~\ref{alg:mvp}}
\label{sec:fs-mvp-appendix}
The above analysis provides asymptotic convergence guarantee of Algorithm~\ref{alg:mvp}, however, it is desirable to know the sample complexity of the algorithm in real applications. Motivated by offering RL practitioners confidence in applying the algorithm, we then present the sample complexity analysis with Option II described in Algorithm~\ref{alg:mvp}, i.e., the stepsizes are set to be a constant, and the output is randomly selected from $\{x_1,\cdots,x_N\}$ with a discrete uniform distribution.
This is a standard strategy for nonconvex stochastic optimization approaches \citep{cd:sbmd:dang2015}.
With these algorithmic refinements, we are ready to present the finite-sample analysis as follows. It should be noted that this proof is a special case of the general stochastic nonconvex BSG algorithm analysis provided in Appendix~\ref{sec:proof:thm:cyclicbcd}.
%\vspace{-0.64cm}

\begin{proof}[Proof of Theorem \ref{thm:cyclicbcd2}]

The proof of finite-sample analysis starts from the similar idea with asymptotic convergence. The following analysis follows from Eq.~\eqref{eq:rearr} in the proof of Theorem \ref{thm:asyn}, but we are using stepsizes $\{\beta_t^{\theta}\}$, $\{\beta_t^{y}\}$ are constants which satisfy $2\beta_t^{\min} > L(\beta_t^{\max})^2$ for $t = 1,\cdots,N$ in this proof.

Summing Eq.~\eqref{eq:rearr} over $t$, we have
\begin{align}
\label{eq:finalsum}
&\sum_{t = 1}^N \sum_{i = 1}^2(\beta_t^{\min} - \dfrac{L}{2}(\beta_t^{\max})^2)\mathbb{E}[\|g_t^i\|_2^2]  \\
% = & \sum_{t = 1}^N (\beta_t^{\min} - \dfrac{L}{2}(\beta_t^{\max})^2) \mathbb{E}[\|\nabla\hat{f}_\lambda(x_t) \|_2^2]  \\
\leq&  \hat{f}_\lambda^* - \hat{f}_\lambda(x_1) +\sum_{t = 1}^N [ (\beta_t^{\max})^2 A G + L(\beta_t^{\max})^2\sigma^2 + 2L\beta_t^{\max}(\beta_t^{\max} + L(\beta_t^{\max})^2)(3 \sigma^2 + 2 G^2)].
\end{align}

% We consider an equivalent expression of the output in Option II as
% \begin{equation}
% \Pr(z = t) = \dfrac{\beta_t^{\max}\displaystyle(1 - \frac{L}{2}\beta_t^{\max})}{\displaystyle\sum_{k = 1}^N \beta_t^{\max}(1 - \frac{L}{2}\beta_t^{\max})}, \ t = 1,\dotsc, N
% \label{eq:randompick}
% \end{equation}
% Note that since $\beta_t^{\max}$ is a constant, it reduces to $\Pr(z = t) = 1/N$. Then, multiply $1 / \sum_{k = 1}^N \beta_t^{\max}(1 - \frac{L}{2}\beta_t^{\max})$ on both side of Eq.~\eqref{eq:finalsum}, we obtain
% \begin{align}
% & \sum_{t = 1}^N \dfrac{\beta_t^{\max}\displaystyle(1 - \frac{L}{2}\beta_t^{\max})\mathbb{E}[\|\nabla\hat{f}_\lambda(x_t) \|_2^2]}{\displaystyle\sum_{k = 1}^N \beta_t^{\max}(1 - \frac{L}{2}\beta_t^{\max})}  \\
% = & \mathbb{E}[\|\nabla\hat{f}_\lambda(x_z) \|_2^2] \leq  \dfrac{\hat{f}_\lambda^* - \hat{f}_\lambda(x_1) + N (\beta_t^{\max})^2 C}{N (\beta_t^{\max} - \dfrac{L}{2}(\beta_t^{\max})^2)},
% \end{align}
% where the equality follows from Eq.~\eqref{eq:randompick}
% \begin{align}
% C = A M_\rho + L\sigma^2 + 2L(2 + L\beta_t^{\max})(3 \sigma^2 + 2 M_\rho^2).
% \end{align}
% Then, rearranging the terms, using the similar technique with \citet{cd:sbmd:dang2015} (detailed in Section 4), we obtain Eq.~\eqref{eq:thm2} and \eqref{eq:C}. This completes the proof.

Next, we bound $\mathbb{E}[\|\nabla\hat{f}_\lambda(x_t) \|_2^2]$ using $\mathbb{E}[\|g_t^\theta\|_2^2 + \|g_t^y\|_2^2]$
\begin{align}
\mathbb{E}[\|\nabla\hat{f}_\lambda(x_t) \|_2^2] = & \mathbb{E}[\|\nabla_\theta\hat{f}_\lambda(x_t) \|_2^2 + \|\nabla_y\hat{f}_\lambda(x_t) \|_2^2]  \\
\leq &\mathbb{E}[\|\nabla_\theta\hat{f}_\lambda(\theta_t, y_{t}) - \nabla_\theta\hat{f}_\lambda(\theta_t, y_{t + 1}) + \nabla_\theta\hat{f}_\lambda(\theta_t, y_{t + 1})\|_2^2 + \|g_t^y\|_2^2]  \\
\leq &\mathbb{E}[\|\nabla_\theta\hat{f}_\lambda(\theta_t, y_{t}) - g_t^\theta\|_2^2 + \|g_t^\theta\|_2^2  + 2\langle\nabla_\theta\hat{f}_\lambda(\theta_t, y_{t}) - g_t^\theta,  g_t^\theta\rangle + \|g_t^y\|_2^2]  \\
\leq &\mathbb{E}[\|\nabla_\theta\hat{f}_\lambda(\theta_t, y_{t}) - g_t^\theta\|_2^2 + 2\langle\nabla_\theta\hat{f}_\lambda(\theta_t, y_{t}) - g_t^\theta,  g_t^\theta\rangle] + \mathbb{E}[\|g_t^\theta\|_2^2 + \|g_t^y\|_2^2]  \\
\leq & \mathbb{E}[L^2 \|y_{t + 1} - y_{t}\|_2^2 + 2L\|y_{t + 1} - y_{t}\|_2\cdot\|g_t^\theta\|_2] + \mathbb{E}[\|g_t^\theta\|_2^2 + \|g_t^y\|_2^2]  \\
\leq & \mathbb{E}[L^2 (\beta_{t}^y)^2\|\tilde{g}_t^y\|_2^2 + 2L\beta_{t}^y\|\tilde{g}_t^y\|_2\cdot\|g_t^\theta\|_2] + \mathbb{E}[\|g_t^\theta\|_2^2 + \|g_t^y\|_2^2]  \\
\leq & L^2 (\beta_{t}^y)^2(G^2 + \sigma^2) + 2L\beta_{t}^y\mathbb{E}[\|\tilde{g}_t^y\|_2\cdot\|g_t^\theta\|_2] + \mathbb{E}[\|g_t^\theta\|_2^2 + \|g_t^y\|_2^2]  \\
\leq & L^2 (\beta_{t}^y)^2(G^2 + \sigma^2) + L\beta_{t}^y\mathbb{E}[\|\tilde{g}_t^y\|_2^2 + \|g_t^\theta\|_2^2] + \mathbb{E}[\|g_t^\theta\|_2^2 + \|g_t^y\|_2^2]  \\
\leq & L^2 (\beta_{t}^{\max})^2(G^2 + \sigma^2) + L\beta_{t}^{\max}(2G^2 + \sigma^2) + \mathbb{E}[\|g_t^\theta\|_2^2 + \|g_t^y\|_2^2].
\label{eq:newbd}
\end{align}

Then, combine Eq.~\eqref{eq:newbd} with Eq.~\eqref{eq:finalsum}
\begin{align}
& \sum_{t = 1}^N (\beta_t^{\min} - \dfrac{L}{2}(\beta_t^{\max})^2) \mathbb{E}[\|\nabla\hat{f}_\lambda(x_t) \|_2^2]  \\
\leq & \sum_{t = 1}^N (\beta_t^{\min} - \dfrac{L}{2}(\beta_t^{\max})^2) \mathbb{E}[L^2 (\beta_{t}^{\max})^2(G^2 + \sigma^2) + L\beta_{t}^{\max}(2G^2 + \sigma^2) + \|g_t^\theta\|_2^2 + \|g_t^y\|_2^2]  \\
\leq & \hat{f}_\lambda^* - \hat{f}_\lambda(x_1) +\sum_{t = 1}^N [(\beta_t^{\min} - \dfrac{L}{2}(\beta_t^{\max})^2)(L^2 (\beta_{t}^{\max})^2(G^2 + \sigma^2)  \\
& + L\beta_{t}^{\max}(2G^2 + \sigma^2)) + (\beta_t^{\max})^2 A G + L(\beta_t^{\max})^2\sigma^2  + 2L\beta_t^{\max}(\beta_t^{\max} + L(\beta_t^{\max})^2)(3 \sigma^2 + 2 G^2)]  \\
\leq&  \hat{f}_\lambda^* - \hat{f}_\lambda(x_1) + (\beta_t^{\max})^2\sum_{t = 1}^N [(1 - \dfrac{L}{2}\beta_t^{\max})(L^2 \beta_{t}^{\max}(G^2 + \sigma^2) + L(2M_{\rho}^2 + \sigma^2)) + A G + L\sigma^2  \\
& + 2L(1 + L\beta_t^{\max})(3 \sigma^2 + 2 G^2)].
\end{align}

Rearrange it, we obtain
\begin{align}
\mathbb{E}[\|\nabla\hat{f}_\lambda(x_z) \|_2^2] \leq \dfrac{\hat{f}_\lambda^* - \hat{f}_\lambda(x_1) + N (\beta_t^{\max})^2C}{N (\beta_t^{\min} - \frac{L}{2}(\beta_t^{\max})^2)},
\end{align}
where
\begin{align}
C = & (1 - \dfrac{L}{2}\beta_t^{\max})(L^2 \beta_{t}^{\max}(G^2 + \sigma^2) + L(2G^2 + \sigma^2)) + A G + L\sigma^2 + 2L(1 + L\beta_t^{\max})(3 \sigma^2 + 2 G^2).  
\end{align}

\end{proof}

% Especially, if we take $\beta_t^{\max} = \frac{1}{\sqrt{N}}$, we can simplify our bound as $O\left(1/\sqrt{N}\right).$

\begin{remark}
In Theorem \ref{thm:cyclicbcd2}, we have proven the finite-sample analysis of Algorithm~\ref{alg:mvp} with Option II, i.e., constant stepsizes and randomly picked solution.
Note that the error bound in Eq.~\eqref{eq:thm2} can be simplified as $\mathcal O(1 / (N \beta_t^{\min})) + \mathcal O(\beta_t^{\max})$.
Especially, if $\beta_t^{\max}=\beta_t^{\min}=\beta_t^\theta = \beta_t^y$ are set to be $\Theta(1 / \sqrt{N})$, then the convergence rate of Option II in Algorithm~\ref{alg:mvp} is $\mathcal O(1/\sqrt{N})$. 
\end{remark}
%

%%%%%%%%%%%%%%%%%%%%%%%%%%%%%%%%
\section{Proofs in Convergence Analysis of Nonconvex BSG}
\label{sec:probsg}

This section includes proof of Lemma \ref{lem:err_bound} and Algorithm \ref{thm:cyclicbcd}.
We first provide the pseudo-code for nonconvex BSG method as Algorithm \ref{alg:c_sbcd}.

\begin{algorithm}[thb]
\caption{The nonconvex BSG Algorithm}
\label{alg:c_sbcd}
{\bfseries Input:} Initial point $x_1 \in \mathbb{R}^n$, stepsizes $\{ \beta_t^i: i = 1,\cdots,b \}_{k = 1}^{\infty}$, positive integers $\{m_t\}_{k = 1}^\infty$ that indicate the mini-batch sizes, and iteration limit $N$.
\begin{algorithmic}[1]
\FOR{$k = 1,2,\dotsc,N$}
    \STATE Sample mini batch $\Xi_t = \{\xi_{t,1},\xi_{t,2},\dotsc,\xi_{t,m_t}\}$.
    \STATE Specify update order $\pi_t^i = i$, $i = 1,\cdots,b$, or randomly shuffle $\{i = 1,\cdots,b\}$ to $\{\pi_t^{1},\pi_t^{2},\dotsc,\pi_t^{b}\}$.
    \FOR{$i = 1,2,\dotsc,b$}
        \STATE Compute the stochastic partial gradient for the $\pi_t^i$th block as
            \begin{align}
                \tilde g_t^i = \frac{1}{m_t} \sum_{j = 1}^{m_t} \nabla_{x^{\pi_t^i}}F(x_{t + 1}^{\pi_t^{< i}}, x_{t}^{\pi_t^{\geq i}}; \xi_{t,j}).
            \end{align}
        \STATE Update $\pi_t^i$th block
            \begin{align}
                x_{t + 1}^{\pi_t^i} = x_{k}^{\pi_t^i} - \beta_t^i \tilde g_t^i.
            \end{align}
    \ENDFOR
\ENDFOR
\STATE Return $\bar x_{N} = x_{z}$ randomly according to
    \begin{align}
        \Pr(z = t) = \frac{\beta_t^{\min} - \frac{ L }{2}(\beta_t^{\max})^2}{\sum_{t = 1}^N (\beta_t^{\min} - \frac{ L }{2}(\beta_t^{\max})^2)}, \ t = 1,\dotsc, N.
    \label{eq:randompick}
    \end{align}
\end{algorithmic}
\end{algorithm}

\subsection{Proof of Lemma \ref{lem:err_bound}}
\label{sec:prooflem:err_bound}
\begin{proof}[Proof of Lemma \ref{lem:err_bound}] 
%[Proof of Lemma \ref{lem:err_bound}] 
We prove Lemma \ref{lem:err_bound} for the case of discrete $\xi$, note that the proof still holds for the case of continuous $\xi$ just by using probability density function to replace probability distribution.
Without the loss of generality, we assume a fixed update order in Algorithm \ref{alg:c_sbcd}:
$
\pi_{t}^{i} = i,
$ 
for all $i$ and $t$. Let $\Xi_t = \{\xi_{k,1},\xi_{k,2},\dotsc,\xi_{k,m_t}\}$ be any mini-batch samples in the $t$-th iteration. Let $\tilde g_{\Xi_{t},t}^i = \frac{1}{m_t} \sum_{j = 1}^{m_t} \nabla_{x^{i}}F(x_{t + 1}^{< i}, x_{t}^{\geq i}; \xi_{t,j})$ and $g_{\Xi_{t},t}^i = \nabla_{x^{i}}f(x_{t + 1}^{< i}, x_{t}^{\geq i})$, and $x_{\Xi_{t},t + 1}^{i} = x_{t}^{i} - \beta_t^i \tilde g_{\Xi_{t},t}^i$. Then we have
\begin{align}
\mathbb{E}[\tilde g_{\Xi_{t},t}^i | \mathbf \Xi_{[t - 1]}] = & \mathbb E_{\Xi_{t}} \left[\frac{1}{m_t} \sum_{j = 1}^{m_t} \nabla_{x^{i}}F(x_{\Xi_{t},t + 1}^{< i}, x_{t}^{\geq i}; \xi_{t,j}) \right]\\
= & \sum_{\xi_{1},\dotsc,\xi_{m_t}}\Pr(\Xi_t = \{\xi_{1},\xi_{2},\dotsc,\xi_{m_t}\}) \frac{1}{m_t} \sum_{j = 1}^{m_t} \nabla_{x^i}F(x_{\Xi_{t},t + 1}^{< i}, x_{t}^{\geq i}; \xi_{j}),
\label{eq:exp_tg}
\end{align}
and
\begin{align}
\mathbb{E}[g_{\Xi_{t},t}^i | \mathbf \Xi_{[t - 1]}] = & \mathbb E_{\Xi_{t}} \left[\nabla_{x^{i}}f(x_{\Xi_{t},t + 1}^{< i}, x_{t}^{\geq i}) \right]\\
=& \sum_{\xi'_{1},\dotsc,\xi'_{m_t}} \Pr(\Xi'_t = \{\xi'_{1},\xi'_{2},\dotsc,\xi'_{m_t}\}) \nabla_{x^{i}}f(x_{\Xi'_{t},t + 1}^{< i}, x_{t}^{\geq i})\\
=& \sum_{\xi'_{1},\dotsc,\xi'_{m_t}} \Pr(\Xi'_t = \{\xi'_{1},\xi'_{2},\dotsc,\xi'_{m_t}\}) \sum_{\xi_l}\Pr(\xi = \xi_l)  \nabla_{x^{i}}F(x_{\Xi'_{t},t + 1}^{< i}, x_{t}^{\geq i}; \xi_l)\\
=& \sum_{\xi'_{1},\dotsc,\xi'_{m_t}} \Pr(\Xi'_t = \{\xi'_{1},\xi'_{2},\dotsc,\xi'_{m_t}\}) \sum_{\xi_{1},\dotsc,\xi_{m_t}}  \Pr(\Xi_t = \{\xi_{1},\xi_{2},\dotsc,\xi_{m_t}\}) \\
& \frac{1}{m_t} \sum_{j = 1}^{m_t} \nabla_{x^{i}}F(x_{\Xi'_{t},t + 1}^{< i}, x_{t}^{\geq i}; \xi_j).
\label{eq:exp_g}
\end{align}

% \begin{align}
% & \nabla_{x^{i}}f(x_{\Xi'_{t},t + 1}^{< i}, x_{t}^{\geq i})\\
% = & \sum_{\xi_l}\Pr(\xi = \xi_l)  \nabla_{x^{i}}F(x_{\Xi'_{t},t + 1}^{< i}, x_{t}^{\geq i}; \xi_l)\\
% = & \sum_{\xi_{1},\xi_{2},\dotsc,\xi_{m_t}}  \left(\prod_{j = 1}^{m_t}\Pr(\xi_{l,j} = \xi_{j})\right) \frac{1}{m_t} \sum_{j = 1}^{m_t} \nabla_{x^{i}}F(x_{\Xi'_{t},t + 1}^{< i}, x_{t}^{\geq i}; \xi_j)
% \end{align}

Combine Eq.~\eqref{eq:exp_tg} and Eq.~\eqref{eq:exp_g}, we can obtain the expectation of $\Delta_t^i$ as
\begin{align}
\mathbb{E}[\Delta_t^i | \mathbf \Xi_{[t - 1]}] = &\mathbb{E}[\tilde g_t^i - g_t^i | \mathbf \Xi_{[t - 1]}] \\
=& \sum_{\xi'_{1},\dotsc,\xi'_{m_t}} \Pr(\Xi'_t = \{\xi'_{1},\xi'_{2},\dotsc,\xi'_{m_t}\}) \sum_{\xi_{1},\dotsc,\xi_{m_t}}  \Pr(\Xi_t = \{\xi_{1},\xi_{2},\dotsc,\xi_{m_t}\}) \\
& \frac{1}{m_t} \sum_{j = 1}^{m_t} (\nabla_{x^{i}}F(x_{\Xi_{t},t + 1}^{< i}, x_{t}^{\geq i}; \xi_j) - \nabla_{x^{i}}F(x_{\Xi'_{t},t + 1}^{< i}, x_{t}^{\geq i}; \xi_j)) \\
=& \sum_{\xi'_{1},\dotsc,\xi'_{m_t}} \sum_{\xi_{1},\dotsc,\xi_{m_t}} \Pr(\Xi'_t = \{\xi'_{1},\xi'_{2},\dotsc,\xi'_{m_t}\}) \Pr(\Xi_t = \{\xi_{1},\xi_{2},\dotsc,\xi_{m_t}\}) \\
& \frac{1}{m_t} \sum_{j = 1}^{m_t} (\nabla_{x^{i}}F(x_{\Xi_{t},t + 1}^{< i}, x_{t}^{\geq i}; \xi_j) - \nabla_{x^{i}}F(x_{\Xi'_{t},t + 1}^{< i}, x_{t}^{\geq i}; \xi_j)).
\label{eq:exp_delta}
\end{align}

Note that, since the objection function is Lipschitz smoothness, $F(x; \xi)$ is also Lipschitz smoothness if $\Pr(\xi) > 0$, and we use $L$ to denote the maximum Lipschitz constant for all $\Pr(\xi) > 0$. Similarly, we can also obtain the gradient of $F(x; \xi)$ is also bounded using same analysis, and we use $G$ to denote the maximum bound for all $\Pr(\xi) > 0$. Using these two fact, we have
\begin{align}
& \|\nabla_{x^{i}}F(x_{\Xi_{t},t + 1}^{< i}, x_{t}^{\geq i}; \xi_j) - \nabla_{x^{i}}F(x_{\Xi'_{t},t + 1}^{< i}, x_{t}^{\geq i}; \xi_j)\|_2 \\
\leq & L \|x_{\Xi_{t},t + 1}^{< i} - x_{\Xi'_{t},t + 1}^{< i}\|_2\\
\leq &\sum_{l < i}  L \|x_{\Xi_{t},t + 1}^{l} - x_{\Xi'_{t},t + 1}^{l}\|_2\\
\leq & \sum_{l < i}  L\beta_t^l \|\tilde g_{\Xi_{t},t}^l - \tilde g_{\Xi'_{t},t}^l\|_2\\
\leq &  2 L b G \beta_t^{\max}.
\label{eq:exp_delta_bound}
\end{align}

Combine Eq.~\eqref{eq:exp_delta} and Eq.~\eqref{eq:exp_delta_bound}, we complete the proof as
\begin{align}
& \mathbb{E}[\Delta_t^i | \mathbf \Xi_{[t - 1]}] \\
\leq & \sum_{\xi'_{1},\dotsc,\xi'_{m_t}} \sum_{\xi_{1},\dotsc,\xi_{m_t}} \Pr(\Xi'_t = \{\xi'_{1},\xi'_{2},\dotsc,\xi'_{m_t}\}) \Pr(\Xi_t = \{\xi_{1},\xi_{2},\dotsc,\xi_{m_t}\}) 2 L b G \beta_t^{\max}\\
= & 2 L b G \beta_t^{\max},
\end{align}
where the last equation follows from the Law of total probability. This completes the proof.
\end{proof}

% \begin{remark}
% In Assumption \ref{asm:bderr}, we assume that the stochastic gradient is not unbiased,
% %\noteb{this is the biggest problem right now. two problems
% %1. we do not identify why we get unbiased estimates here. (r we getting unbiased? or we are still getting biased, but with this condition, we can still have sample complexity analysis)
% %2. the bias stuff will confuse RL readers, make them tend to think we are introducing bias.
% %
% %So I suggest we rewrite it more clearly, mention it a little bit (very briefly) in the main paper, and put the details to the appendix.}
% and we have a bounded $\mathbb{E}[\Delta_k^i | \mathbf \Xi_{[k - 1]}]$ in \eqref{eq:stocbvarbd}. The stochastic gradient is commonly assumed to be unbiased and lots of other stochastic programming literature. However, it fails to hold in Algorithm \ref{alg:c_sbcd}. That is because the block update in Algorithm \ref{alg:c_sbcd} are Gauss-Seidel, so that the gradient error $\Delta_k^i$ typically has a nonlinear dependence on $x_{k + 1}^j$ for all $j < i$, which is dependant on $\mathbf \Xi_k$.
% \end{remark}

%%%%%%%%%%%%%%%%%%%%%%%
\subsection{Proof of Theorem \ref{thm:cyclicbcd}}
\label{sec:proof:thm:cyclicbcd}

Let $\beta_t^{\max} \coloneqq \max_i \beta_t^{i}$, $\beta_t^{\min} \coloneqq \min_i \beta_t^{i}$. To establish the convergence rate analysis, we start with Lemma~\ref{lem:vecdoterr}.
\begin{lemma}
\label{lem:vecdoterr}
% Assumption \ref{asm:bderr} holds and 
Let $u_k$ be a random vector that only depends on $\mathbf \Xi_{[t - 1]}$. If $u_t$ is independent of $\Delta_t^i$, then
\begin{align}
\mathbb E [\langle u_t, \Delta_t^i \rangle] \leq A \beta_t^{\max} \mathbb E [\|u_t\|_2],
\end{align}
where $A$ is defined in Eq~\eqref{eq:stocerrbd}.
\end{lemma}
Now, it is ready to discuss the main convergence properties of the nonconvex Cyclic SBCD algorithm (Algorithm \ref{alg:c_sbcd}) and provide the rate of convergence for that.
\begin{proof}[Proof of Lemma \ref{lem:vecdoterr}]
We can obtain the result of Lemma \ref{lem:vecdoterr} by follows
\begin{align}
\mathbb{E} [\langle u_t, \Delta_t^i \rangle] = & \mathbb{E}_{\mathbf \Xi_{[t - 1]}} \left[\mathbb{E}[ \langle u_t, \Delta_t^i \rangle | \mathbf \Xi_{[t - 1]}]\right]\\
\overset{\text{(a)}} = & \mathbb{E}_{\mathbf \Xi_{[t - 1]}} \left[ \langle \mathbb{E}[ u_t | \mathbf \Xi_{[t - 1]}], \mathbb{E}[ \Delta_t^i | \mathbf \Xi_{[t - 1]}] \rangle \right]\\
\leq & \mathbb{E}_{\mathbf \Xi_{[t - 1]}} \left[ \| \mathbb{E}[ u_t | \mathbf \Xi_{[t - 1]}]\|_2 \cdot \|\mathbb{E}[ \Delta_t^i | \mathbf \Xi_{[t - 1]}] \|_2 \right]\\
\leq & A \beta_t^{\max} \mathbb{E}_{\mathbf \Xi_{[t - 1]}} \left[ \| \mathbb{E}[ u_t | \mathbf \Xi_{[t - 1]}]\|_2 \right]\\
\overset{\text{(b)}} \leq & A \beta_t^{\max} \mathbb{E} \left[ \|u_t\|_2 \right],
\end{align}
where (a) follows from the conditional independence between $u_t$ and $\Delta_t^i$, and (b) follows from Jensen's inequality.
\end{proof}

\begin{proof}[Proof of Theorem \ref{thm:cyclicbcd}]
From the Lipschitz smoothness, it holds that
\begin{align}
&f(x_{t + 1}^{\leq i}, x_{t}^{> i}) - f(x_{t + 1}^{< i}, x_{t}^{\geq i}) \\
\leq & \langle g_t^i, x_{t + 1}^i - x_{t}^i \rangle + \frac{L}{2}\|x_{t + 1}^i - x_{t}^i \|_2^2 \\
= & -\beta_t^i \langle g_t^i, \tilde g_t^i \rangle + \frac{L}{2}(\beta_t^i)^2\|\tilde g_t^i\|_2^2  \\
= & -(\beta_t^i - \frac{L}{2}(\beta_t^i)^2)\|g_t^i\|_2^2 + \frac{L}{2}(\beta_t^i)^2\|\Delta_t^i\|_2^2 - (\beta_t^i - L(\beta_t^i)^2)\langle g_t^i,\Delta_t^i \rangle
\label{eq:lipsm}
\end{align}
where all the equations follow the definition of $\Delta_t^i$ and the update law of Algorithm~\ref{alg:c_sbcd}.

Summing Eq.~\eqref{eq:lipsm} over $i$, then we obtain
\begin{align}
& f(x_{t + 1}) -  f(x_{t}) \leq -\sum_{i = 1}^b(\beta_t^i - \frac{ L }{2}(\beta_t^i)^2)\|g_t^i\|_2^2 + \sum_{i = 1}^b \frac{ L }{2}(\beta_t^i)^2\|\Delta_t^i\|_2^2 - \sum_{i = 1}^b(\beta_t^i -  L (\beta_t^i)^2)\langle g_t^i, \Delta_t^i \rangle.
\label{eq:bcdsumming}
\end{align}

Use Lemma \ref{lem:vecdoterr}, we also have the following fact,
\begin{align}
\label{eq:expbound}
\mathbb{E}[\langle g_t^i, \Delta_t^i \rangle] \leq \beta^{\max}_t A G.
\end{align}

Taking expectation over Eq.~\eqref{eq:bcdsumming}, we have
\begin{align}
& \mathbb{E}[ f(x_{t + 1})] - \mathbb{E}[ f(x_{t})]  \\
\leq & -\sum_{i = 1}^b(\beta_t^i - \frac{ L }{2}(\beta_t^i)^2)\mathbb{E}[\|g_t^i\|_2^2] + \sum_{i = 1}^b(\beta_t^i - L(\beta_t^i)^2)\beta^{\max}_t A G + \sum_{i = 1}^b\frac{ L }{2}(\beta_t^i)^2\mathbb{E}[\|\Delta_t^i\|_2^2]  \\
\leq & -(\beta_t^{\min} - \frac{ L }{2}(\beta_t^{\max})^2) \sum_{i = 1}^b \mathbb{E}[\|g_t^i\|_2^2] + \sum_{i = 1}^b\left( (\beta^{\max}_t)^2 A G + \frac{ L }{2}(\beta_t^i)^2\sigma^2 \right),
\label{eq:bcdexp}
\end{align}
where the first inequality follows from Eq.~\eqref{eq:expbound}, and the second inequality follows from the boundedness of $\mathbb{E}[\|\Delta_t^i\|_2]$ and $\mathbb{E}[\|g_t^i\|_2]$.
% $\mathbb{E}[\|\nabla_{x^i}\hat{f}_{\lambda}(x_{t + 1}^{<i}, x_{t}^{\geq i})\|_2]$ (i.e. the boundedness of $\mathbb{E}[\|g_t^i\|_2]$).

Rearranging Eq.~\eqref{eq:bcdexp}, we obtain
\begin{align}
&(\beta_t^{\min} - \frac{ L }{2}(\beta_t^{\max})^2) \sum_{i = 1}^b \mathbb{E}[\|g_t^i\|_2^2] \leq \mathbb{E}[ f(x_{t})] - \mathbb{E}[ f(x_{t + 1})] + \sum_{i = 1}^b\left( (\beta^{\max}_t)^2 A G + \frac{ L }{2}(\beta_t^i)^2\sigma^2 \right).
\label{eq:bcdrearr}
\end{align}

Also, we have
\begin{align}
\mathbb{E}[\|\nabla_{x^i} f(x_t)\|_2] \leq & \mathbb{E}[\|\nabla_{x^i} f(x_t) - g_t^i\|_2] + \mathbb{E}[\|g_t^i\|_2] \\
\overset{\text{(a)}} \leq &  L  \mathbb{E}[\|x_{t + 1}^{<i} - x_t^{<i}\|_2] + \mathbb{E}[\|g_t^i\|_2] \\
\overset{\text{(b)}} \leq &  L   \mathbb{E}\left[ \sqrt{\sum_{j < i} \| \beta_t^j \tilde g_t^j \|_2^2 } \right] + \mathbb{E}[\|g_t^i\|_2] \\
\leq &  L  \beta_t^{\max}  \mathbb{E}\left[ \sqrt{\sum_{j < i} \|\tilde g_t^j \|_2^2 } \right] + \mathbb{E}[\|g_t^i\|_2] \\
\overset{\text{(c)}} \leq &  L  \beta_t^{\max}  \sqrt{\mathbb{E}[\sum_{j < i} \|\tilde g_t^j \|_2^2 ]} + \mathbb{E}[\|g_t^i\|_2] \\
\overset{\text{(d)}} \leq &  L  \beta_t^{\max}  \sqrt{\sum_{j < i} (G^2 + \sigma^2) } + \mathbb{E}[\|g_t^i\|_2],
\label{eq:gradieq}
\end{align}
where (a) follows from the Lipschitz smoothness of $f$, (b) follows from $x_{t + 1}^j = x_t^j - \beta_t^j \tilde g_t^j$, (c) follows from Jenson's inequality, and (d) follows from the boundedness of gradient and boundedness of variance.

Summing Eq.~\eqref{eq:gradieq} over $i$, we can obtain
\begin{align}
& \mathbb{E}[\|\nabla f(x_t) \|_2^2] = \sum_{i = 1}^{b}\mathbb{E}[\|\nabla_{x^i} f(x_t)\|_2]\leq \sum_{i = 1}^{b}  L  \beta_t^{\max} \sqrt{\sum_{j < i} (G^2 + \sigma^2) } + \sum_{i = 1}^{b} \mathbb{E}[\|g_t^i\|_2].
\label{eq:gradbd}
\end{align}

Combine Eq.~\eqref{eq:gradbd} with Eq.~\eqref{eq:bcdrearr}, we can obtain
\begin{align}
&(\beta_t^{\min} - \frac{ L }{2}(\beta_t^{\max})^2) \mathbb{E}[\|\nabla f(x_t) \|_2^2]\\
\leq & (\beta_t^{\min} - \frac{ L }{2}(\beta_t^{\max})^2) \sum_{i = 1}^{b}  L  \beta_t^{\max} \sqrt{\sum_{j < i} (G^2 + \sigma^2) } + (\beta_t^{\min} - \frac{ L }{2}(\beta_t^{\max})^2) \sum_{i = 1}^{b} \mathbb{E}[\|g_t^i\|_2]\\
\leq & \mathbb{E}[ f(x_{t})] - \mathbb{E}[ f(x_{t + 1})] + (\beta_t^{\min} - \frac{ L }{2}(\beta_t^{\max})^2) \sum_{i = 1}^{b}  L  \beta_t^{\max} \sqrt{\sum_{j < i} (G^2 + \sigma^2) }\\
& + \sum_{i = 1}^b\left( (\beta^{\max}_t)^2 A G + \frac{ L }{2}(\beta_t^i)^2\sigma^2 \right),
\label{eq:eachite}
\end{align}
where the first inequality follows from substituting Eq.~\eqref{eq:gradbd} into the left-hand side of Eq.~\eqref{eq:bcdrearr}, and the second inequality follows from substituting Eq.~\eqref{eq:gradbd} into the right-hand side of Eq.~\eqref{eq:bcdrearr}.

Summing Eq.~\eqref{eq:eachite} over $t$, we have
\begin{align}
&\sum_{t = 1}^N (\beta_t^{\min} - \frac{ L }{2}(\beta_t^{\max})^2) \mathbb{E}[\|\nabla f(x_t) \|_2^2] \\
\leq & f(x_{1}) - f(x^*) \\
& + \sum_{t = 1}^N \left[(\beta_t^{\min} - \frac{ L }{2}(\beta_t^{\max})^2) \sum_{i = 1}^{b}  L  \beta_t^{\max} \sqrt{\sum_{j < i} (G^2 + \sigma^2) } + \sum_{i = 1}^b\left( (\beta^{\max}_t)^2 A G + \frac{ L }{2}(\beta_t^i)^2\sigma^2 \right)\right] \\
\leq & f(x_{1}) - f(x^*) + \sum_{t = 1}^N(\beta_t^{\max})^2 C_t.
\label{eq:bcdfinalsum} 
\end{align}
where $C_t$ is
\begin{align}
C_t = & (1 - \frac{ L }{2}\beta_t^{\max}) \sum_{i = 1}^{b}  L  \sqrt{\sum_{j < i} (G^2 + \sigma^2) } + b\left( A G + \frac{ L }{2} \sigma^2\right).
\end{align}
Using the probability distribution of $R$ given in Eq.~\eqref{eq:randompick}, we completes the proof.
\end{proof}

\begin{proof}[Proof of Corollary \ref{cor:cyclicrate}]
Combine these conditions with Eq.~\eqref{eq:bcdfinalsum}, we have
\begin{align}
&\sum_{t = 1}^N (\beta^{\min} - \frac{ L }{2}(\beta^{\max})^2) \mathbb{E}[\|\nabla f(x_t) \|_2^2] \leq f(x_{1}) - f(x^*) + N(\beta^{\max})^2 C,
\end{align}
where $C$ is
\begin{align}
C = & (1 - \frac{ L }{2}\beta^{\max}) \sum_{i = 1}^{b}  L  \sqrt{\sum_{j < i} (G_j^2 + \sigma^2) } + b\left( A G + \frac{ L }{2} \sigma^2 \right).
\end{align}
Using the probability distribution of $z$ given in Eq.~\eqref{eq:randompick}, we can obtain
\begin{align}
&\mathbb{E}\left[\|\nabla f(x_z)\|_2^2\right] \leq \frac{ f(x_1) - f^* + N (\beta^{\max})^2 C}{N (\beta^{\min} - \frac{ L }{2}(\beta^{\max})^2)}.
\end{align}

Thus, we can reach the rate of convergence of $\mathcal O(1 / \sqrt{N})$ by setting $\beta^{\min} = \beta^{\max} = \mathcal O(1 / \sqrt{N})$.
\end{proof}

%%%%%%%%%%%%%%%%%%%%%%%%%%%%%%%%
\section{RCPG and SGA Algorithm}
\label{sec:rcpg-sga}
\subsection{Randomized Stochastic Block Coordinate Descent Algorithm}
\label{sec:rcpg}

We propose the randomized stochastic block coordinate descent algorithm as  Algorithm~\ref{alg:sbmdpg}. Note that we also use the same notation about gradient from Eq.~\eqref{def:tildeg_y} and Eq.~\eqref{def:tildeg_theta} with very a tiny difference in practical, where $y_{t + 1} = y_{t}$ in Eq.~\eqref{def:tildeg_theta}.

\begin{algorithm}[htb!]
\caption{Risk-Sensitive Randomized Coordinate Descent Policy Gradient (RCPG)}
\label{alg:sbmdpg}
\centering
\begin{algorithmic}[1]
\STATE {\bfseries Input:} Stepsizes $\{\beta_t^\theta\}$ and $\{\beta_t^y\}$, let $\beta_t^{\max} = \max\{\beta_t^\theta, \beta_t^y\}$.\\
\textbf{Option I:} $\{\beta_t^\theta\}$ and $\{\beta_t^y\}$ satisfy the Robbins-Monro condition.\\
\textbf{Option II:} $\beta_t^\theta$ and $\beta_t^y$ are set to be constants.
\FOR {episode $t=1,\dotsc,N$}
\FOR {time step $k=1,\dotsc,{\tau_t}$}
\STATE Compute $a_k \sim \pi_\theta(a|s_k)$, observe $r_k, s_{k+1}$.
\ENDFOR
\STATE Compute
\begin{align}
{R_t} &= \sum_{k = 1}^{\tau_t}  {{r_k}} \\
{\omega _t}(\theta_t) 
&= \sum _{k=1}^{{\tau_t}}{\nabla_\theta \ln{\pi_{\theta_t}}({a_{k}}|{s_{k}})}.
\end{align}
\STATE Randomly select $i_t \in \{1,2\}$ with distribution $[0.5,0.5]$. If $i_t=1$,
\begin{align}
{y_{t+1}} &= {y_t} + {\beta _t}\left(2R_t + \frac{1}{\lambda} - 2y_t\right),\\
{\theta_{t+1}} &= {\theta _t}.
\end{align}
else
\begin{align}
{y_{t+1}} &= {y_t},\\
{\theta_{t+1}} &= {\theta _t} + {\beta _t}\left( {2y_t{R_t} - {{({R_t})}^2}} \right){\omega _t}({\theta _t}).
\end{align}
\ENDFOR
\STATE {\bfseries Output} $\bar{x}_N$:
\\ \textbf{Option I:} Set  $\bar{x}_N = x_N$.
\\ \textbf{Option II:} Set $\bar{x}_N = x_z$, where $z$ is uniformly drawn from $\{ 1,2, \dotsc ,N\}$.
\end{algorithmic}
\end{algorithm}
%\end{document}
Note that, the main difference between Cyclic SBCD and Randomized SBCD is that: at each iteration, Cyclic SBCD cyclically updates all blocks of variables, and the later updated blocks depending on the early updated blocks; while Randomized SBCD randomly chooses one block of variables to update.

\subsection{Risk-Sensitive Stochastic Gradient Ascent Policy Gradient}
\label{sec:sga}

We also proposed risk-sensitive stochastic gradient Ascent policy gradient as Algorithm \ref{alg:sgapg}.

\begin{algorithm}[htb!]
\caption{Risk-Sensitive Stochastic Gradient Ascent Policy Gradient (SGA)}
\label{alg:sgapg}
\centering
\begin{algorithmic}[1]
\STATE {\bfseries Input:} Stepsizes $\{\beta_t^\theta\}$ and $\{\beta_t^y\}$, let $\beta_t^{\max} = \max\{\beta_t^\theta, \beta_t^y\}$.\\
\textbf{Option I:} $\{\beta_t^\theta\}$ and $\{\beta_t^y\}$ satisfy the Robbins-Monro condition.\\
\textbf{Option II:} $\beta_t^\theta$ and $\beta_t^y$ are set to be constants.
\FOR {episode $t=1,\dotsc,N$}
\FOR {time step $k=1,\dotsc,{\tau_t}$}
\STATE Compute $a_k \sim \pi_\theta(a|s_k)$, observe $r_k, s_{k+1}$.
\ENDFOR
\STATE Compute
\begin{align}
{R_t} &= \sum_{k = 1}^{\tau_t}  {{r_k}} \\
{\omega _t}(\theta_t) 
&= \sum _{k=1}^{{\tau_t}}{\nabla_\theta \ln{\pi_{\theta_t}}({a_{k}}|{s_{k}})}.
\end{align}
\STATE Update parameters,
\begin{align}
{y_{t+1}} &= {y_t} + {\beta _t}\left(2R_t + \frac{1}{\lambda} - 2y_t\right),\\
{\theta_{t+1}} &= {\theta _t} + {\beta _t}\left( {2y_t{R_t} - {{({R_t})}^2}} \right){\omega _t}({\theta _t}).
\end{align}
\ENDFOR
\STATE {\bfseries Output} $\bar{x}_N$:
\\ \textbf{Option I:} Set  $\bar{x}_N = x_N$.
\\ \textbf{Option II:} Set $\bar{x}_N = x_z$, where $z$ is uniformly drawn from $\{ 1,2, \dotsc ,N\}$.
\end{algorithmic}
\end{algorithm}

%%%%%%%%%%%%%%%%%%%%%%%%
\section{Details of the Experiments}
\label{sec:parameters}

The parameter settings for portfolio management domain are as follows: $\tau  = 50$, $r_l = 1.001$, $r_{\text{nl}}^{\text{high}} = 2$, $r_{\text{nl}}^{\text{low}} = 1.1$, $p_{\text{risk}} = 0.05$, $p_{\text{switch}} = 0.1$, $W = 4$, $\eta = 0.2$, startup cash $\$100,000$.

The parameter settings of American-style Option domain are as follows: $K_{\text{put}} = 1$, $K_{\text{call}} = 1.5$, $x_0 = 1.25$, $f_u = 9 / 8$, $f_d = 8 / 9$, $p = 0.45$, $\tau = 20$.

The parameter settings of optimal stopping domain are as follows: $x_0 = 1.25$, $f_u = 2$, $f_d = 0.5$, $p = 0.65$, $\tau = 20$.